\documentclass{article} 
\usepackage{iclr2026_conference,times}

\usepackage{url}
\usepackage{amsmath}
\usepackage{amsfonts}
\usepackage{amsthm}
\usepackage{hyperref}
\usepackage{smile}
\usepackage{subfigure}

\usepackage{tablefootnote}
\usepackage{mathtools}

\usepackage{colortbl}
\definecolor{LightCyan}{rgb}{0.8, 0.9, 1}

\usepackage{amssymb}
\usepackage{pifont}

\ifdefined\final
\usepackage[disable]{todonotes}
\else
\usepackage[textsize=tiny]{todonotes}
\fi
\usepackage{blindtext}
\usepackage[compact]{titlesec}
\usepackage{sidecap}
\titlespacing*{\section}{0pt}{*0.5}{*0.5}
\titlespacing*{\subsection}{0pt}{*0.5}{*0.5}
\titlespacing*{\subsubsection}{0pt}{*0.5}{*0.5}

\title{Towards a Sharp Analysis of Offline Policy Learning for $f$-Divergence-Regularized Contextual Bandits}

\author{Qingyue Zhao$^{1}$\thanks{Equal Contribution}\ \ ,
Kaixuan Ji$^{1}$$^*$, 
Heyang Zhao$^{1}$$^*$,
Tong Zhang$^{2}$,
Quanquan Gu$^{1}$\\
$^1$University of California, Los Angeles\\
$^2$University of Illinois Urbana-Champaign\\
\texttt{\{kaixuanji,zhaoqy24,hyzhao\}@cs.ucla.edu}\\
\texttt{tozhang@illinois.edu,qgu@cs.ucla.edu}
}

\iclrfinalcopy 

\newcommand{\piref}{\pi^{\mathsf{ref}}}
\newcommand{\kl}[2]{\ensuremath{{\mathsf{KL}}\left(#1\|#2\right)}}
\newcommand{\tv}[2]{\ensuremath{{\mathsf{TV}}\left(#1\|#2\right)}}
\newcommand{\confbandit}{\cE}
\newcommand{\confduelbandit}{\cE}

\newcommand{\fls}{{\bar{g}}} 
\newcommand{\fps}{{\hat{g}}} 

\newcommand{\fgt}{{g^*}} 

\newcommand{\fcl}{{\cG}} 

\def \algcb {\text{KL-PCB}}
\def \algcdb {\text{KL-PCDB}}

\def \algfcb {\text{$f$-CB}}
\def \algfcdb {\text{$f$-CDB}}

\newcommand{\pistar}{\pi^*}
\newcommand{\subopt}{\mathrm{SubOpt}}

\newcommand{\objfdiv}{J_{f\mathrm{div}}}
\newcommand{\pistarfdiv}{\pi^*_{f\mathrm{div}}}
\newcommand{\suboptfdiv}{\mathrm{SubOpt}_{f\mathrm{div}}}

\newcommand{\rE}{\mathbb E}

\newcommand{\KL}{\mathsf{KL}}

\newcommand{\la}{\left\langle}
\newcommand{\ra}{\right\rangle}

\newcommand{\E}{\rE}
\newcommand{\minus}{\scalebox{0.75}[1.0]{$-$}}

\usepackage{soul}

\begin{document}

\maketitle

\begin{abstract}
Many offline reinforcement learning algorithms are underpinned by $f$-divergence regularization, but their sample complexity \emph{defined with respect to regularized objectives} still lacks tight analyses, especially in terms of concrete data coverage conditions. 
In this paper, we study the exact concentrability requirements to achieve the $\tilde{\Theta}(\epsilon^{-1})$ sample complexity for offline $f$-divergence-regularized contextual bandits. 
For reverse Kullback–Leibler (KL) divergence, arguably the most commonly used one, we achieve an $\tilde{O}(\epsilon^{-1})$ sample complexity under single-policy concentrability for the first time via a novel pessimism-based analysis, surpassing existing $\tilde{O}(\epsilon^{-1})$ bound under all-policy concentrability and $\tilde{O}(\epsilon^{-2})$ bound under single-policy concentrability. 
We also propose a near-matching lower bound, demonstrating that a multiplicative dependency on single-policy concentrability is necessary to maximally exploit the curvature property of reverse KL. 
Moreover, for $f$-divergences with strongly convex $f$, to which reverse KL \emph{does not} belong, we show that the sharp sample complexity $\tilde{\Theta}(\epsilon^{-1})$ is achievable even without pessimistic estimation or single-policy concentrability.
We further corroborate our theoretical insights with numerical experiments and extend our analysis to contextual dueling bandits. We believe these results take a significant step towards a comprehensive understanding of objectives with $f$-divergence regularization.
\end{abstract}

\section{Introduction}\label{sec:intro}

Due to the data-hungry and instable nature of reinforcement learning (RL), divergences that are {straightforward to estimate via Monte Carlo} or {amenable to constrained optimization} stand out from numerous candidates \citep{renyi1961measures,csiszar1967information,muller1997integral,basseville2013divergence} as regularizers; the former family is typically \emph{$f$-divergence} \citep{renyi1961measures} because any of them is an expectation, for which empirical average is a good proxy~\citep{levine2018reinforcement,levine2020offline}; and the latter class subsumes those with nice positive curvatures (e.g., Bregman divergence \citep{bregman1967relaxation} induced by strongly convex functions). In particular, \emph{Kullback-Leibler (KL) divergence} is the only one at the intersection of $f$-divergence and Bregman divergence \citep[Theorem~5]{jiao2014information}, indicating its theoretical
advantage
among common choices from both computational and statistical aspects. Also, the \emph{KL-regularized RL objective} is arguably the most popular one in practice:
\begin{align}\label{eq:obj-teaser}
    J(\pi) = \EE_{\pi}[r] - \eta^{-1}\KL({\pi}\|{\piref}),
\end{align}
where $r$ is the reward, $\piref$ is a reference policy, $\KL({\pi}\|{\piref})$ is the reverse KL divergence, and $\eta>0$ is the inverse temperature. When $\piref$ is uniform, \Cref{eq:obj-teaser} reduces to the entropy-regularized objective that encourages diverse actions and enhances robustness \citep{williams1992simple,ziebart2008maximum,levine2013guided,levine2016end,haarnoja2018soft,richemond2024offline,liu2024enhancing}. KL regularization has also been widely used in the RL fine-tuning of large language models \citep{ouyang2022training,rafailov2023direct}, where $\piref$ is the base model. Given its widespread use, there has been a surge of interest in understanding the role of KL regularization in RL by both empirical studies \citep{ahmed2019understanding,liu2019regularization} and theoretical analysis \citep{geist2019theory,vieillard2020leverage,kozuno2022kl}. There are also lines of research on KL regularization in online learning \citep{cai2020provably,he2022near,ji2023horizon} and convex optimization \citep{neu2017unified}. However,
most of these works still study the unregularized reward maximization objective, against which the sample complexity is at least $\Omega(\epsilon^{-2})$.\footnote{See \Cref{sec:attempts} for detailed reasons.}

\newcolumntype{g}{>{\columncolor{LightCyan}}c}
\begin{table*}[t!]
\caption{Comparison of sample complexity
bounds for finding $\epsilon$-optimal policy for offline contextual bandits with KL-
and (strongly convex) $f$-divergence regularization.
Constants and $\mathrm{polylog}$ factors are omitted here except the metric entropy $\log \cN$. ``Reverse-KL'' stands for KL-regularized contextual bandits and ``$f$-divergence w/ s.c., $f$'' for the counterpart with
an $\alpha$-strongly convex $f$. The two existing upper bounds are adapted from the implicit form in \citet[Theorem~3.1]{xiong2024iterative} and \citet[Theorem~3.3 and Theorem~4.4]{zhao2024sharp}, of which the detailed adaptions are deferred to \Cref{sec:exist-result}. The relationship between $D^2_{\pi^*}$ and $C^{\pi^*}$ is detailed in \Cref{sec:cb:setup}.}\label{tab:teaser}
\centering
\begin{tabular}{c  c  c  c  g}
\toprule
\multicolumn{2}{c}{Regularizer} & \citet{xiong2024iterative} & \citet{zhao2024sharp} & This work \\
\midrule
\multirow{2}{*}{ Reverse KL} & Upper & $d\epsilon^{-2}$ & $\eta D^2 \epsilon^{-1} \log \cN $ & $\eta D^2_{\pi^*} \epsilon^{-1} \log \cN $ \\
& Lower & - &$\eta \epsilon^{-1} \log \cN $  & $\eta C^{\pi^*} \epsilon^{-1} \log \cN $ \\
\midrule
\multirow{2}{*}{ \shortstack{$f$-divergence \\ w/ s.c. $f$} } & Upper & - & - & $ \alpha^{-1}\eta \epsilon^{-1} \log \cN $ \\
& Lower & - & - & $ \alpha^{-1}\eta \epsilon^{-1} \log \cN $ \\
\bottomrule
\end{tabular}
\vspace*{-0.2in}
\end{table*}

Several recent papers \citep{xiong2024iterative,xie2024exploratory,zhao2024sharp,foster2025good,aminian2025theoretical} switched the focus to analyzing the sub-optimality defined via the regularized objective \eqref{eq:obj-teaser}, under which an $\Omega(\epsilon^{-1})$ sample complexity is possible \citep[Theorem~3.6]{zhao2024sharp}.
However, even restricted to the pure $\iid$ setting, existing analyses in this vein either result in still $\tilde{O}(\epsilon^{-2})$ bounds \citep{xiong2024iterative,xie2024exploratory} or has stringent (local) all-policy concentrability dependencies in their upper bounds \citep{zhao2024sharp,aminian2025theoretical}.\footnote{See \Cref{sec:cb:setup} for details on coverage conditions.} Thus, there are by far no tight bounds in terms of both the dependency of $\epsilon^{-1}$ and data coverage conditions for KL-regularized offline decision making. In addition, all
analyses 
above
set KL as the right target by default; but reverse KL is the $f$-divergence with $f(x) = x\log x$, which is merely convex. Therefore, it is also unknown whether $f$-divergence regularizers with even nicer (e.g., strongly convex) $f$, whose performance against the reward maximization objective
are
provably promising \citep{zhan2022offline,gabbianelli2024importance,huang2024correcting}, can enjoy a better coverage dependency in their sample complexity when
the corresponding regularized objectives serve as the performance metric.
Because data coverage (i.e., concentrability) conditions captures the crucial distributional shift issue in offline RL~\citep{levine2020offline},
the aforementioned perspectives motivate a pivotal open problem:
\begin{center}
    \emph{What is the weakest coverage condition required for} offline learning to be near-optimal with respect to \emph{$f$-divergence-regularized objectives?}
\end{center}
We attack this problem by showing near-optimal sample complexity
with matching concentrability dependencies for two representative subclasses of $f$-divergence.
First, for contextual bandits with KL regularization, we achieve a near-optimal sample complexity guarantee with linear dependence on \emph{single}-policy coverage ratio. Our novel lower bound further indicates that this multiplicative dependency on \emph{single}-policy concentrability is necessary.
Surprisingly, for $f$-divergence with $\alpha$-strongly-convex $f$, we prove nearly matching sample complexity bounds of $\tilde{\Theta}(\alpha^{-1} \eta \epsilon^{-1})$,
eliminating the dependence on coverage for the first time. For the ease of comparison, we adapt existing counterparts under our notation to the offline setting and summarize them
in \Cref{tab:teaser}. 

\subsection{Contributions}

\begin{itemize}[leftmargin=*]
    \item For KL regularization, we propose a pessimism-based algorithm achieving the tight sample complexity under \emph{single}-policy concentrability. We also obtain a lower bound that linearly scale with the density-ratio-based
    single-policy concentrability. Both results strictly improves upon previous works~\citep{zhao2024sharp,foster2025good} in the offline setting, showing that single-policy concentrability is both sufficient and necessary to achieve the $\tilde{\Theta}(\epsilon^{-1})$ sample complexity.
    \item Technically speaking, our analysis exploits the strong convexity of KL and pessimism of the reward estimator, to refine a mean-value-type risk upper bound (\Cref{lem:taylor-expansion}) to its, which in turn leads to a novel moment-based analysis, effectively bypassing the need for uniform control over the discrepancy between any two functions in the function class. To the best of our knowledge, this machinery has not been used in the standard analysis of existing offline RL algorithms and may be of independent interest.
    
    \item For $f$-divergence-regularized objectives with strongly convex $f$, we design a truly lightweight algorithm free of pessimism-based gadgets and still obtain the $\tilde{\Theta}(\epsilon^{-1})$ sample complexity certified by a matching lower bound without coverage conditions.
    
    \item  We verify the statistical rates above in numerical experiments, and demonstrate the versatility of all algorithmic and constructive proof ideas 
    above by extending them to $f$-divergence-regularized contextual dueling bandits (CDBs), achieving similar $\tilde{\Theta}(\epsilon^{-1})$ sample complexity bounds. Moreover, all algorithms are applicable for reward function classes with small metric entropy.
\end{itemize}


\subsection{Key Related Work}\label{sec:key-related-works}

We review two key lines of theoretical progress that are relevant to our algorithm design and analysis. 
\textbf{Pessimism in offline RL.} The principle of pessimism has been underpinning offline RL for both the tabular \citep{rashidinejad2021bridging} and function approximation \citep{jin2021pessimism} settings 
under the name of lower confidence bound (LCB). For contextual bandits, it is behind the adaptively optimal sample complexity analysis \citep{li2022pessimism}.
\citet{shi2022pessimistic} proposed a LCB-based model-free algorithm for tabular RL with near-optimal guarantee.
\citet{jin2021pessimism,xiong2022nearly,di2024pessimistic} utilized LCB in conjunction with the classic least-square value iteration paradigm to derive $\tilde{O}(\epsilon^{-2})$ sample complexity results for model-free RL with function approximation. 
The line of work from \citet{rashidinejad2021bridging,xie2021policy} to \citet{li2024settling} settled the sample complexity
of tabular model-based RL via pessimistic estimators exploiting the variance information.
It is also possible to leverage the idea of pessimism to design model-based algorithms under general function approximation that are at least statistically efficient \citep{xie2021bellman,uehara2021pessimistic,wang2024model}. The principle of pessimism has also been applied in counterfactual empirical risk minimization~\citep{swaminathan2015counterfactual,london2019bayesian} and offline policy learning~\citep{sakhi2023pac,sakhi2024logarithmic}, which are orthogonal to our contributions.


However, in terms of risk decomposition, to the best of our knowledge, none of these pessimism-based analyses really goes beyond the performance difference lemma \citep[Lemma~13]{foster2023foundations} or simulation lemma \citep[Lemma~23]{foster2023foundations};
both of which are not able to capture the strong concavity of KL-regularized objectives even in the bandit setting.
The algorithmic idea of using pessimistic least-square estimators under general function approximation in \citet{jin2021pessimism,di2024pessimistic} is similar to ours, but their sub-optimality gap is bounded by the sum of bonuses, which cannot directly lead to the desired sample complexity of our objective.

\textbf{Offline CDBs.} CDBs \citep{dudik2015contextual} is the contextual extension of dueling bandits
in the classic literature of online learning from pairwise comparisons \citep{yue2012k,zoghi2014relative}.
Since the empirical breakthrough of preference-based RL fine-tuning of LLMs \citep{ouyang2022training}, the theory of offline CDBs has received more attention
under linear function approximation \citep{zhu2023principled,xiong2024iterative} and general  function approximation \citep{zhan2022offline,zhao2024sharp,song2024importance,huang2024correcting}. Preference models without stochastic transitivity \citep{munos2023nash,ye2024theoretical,wu2024self,zhang2024general} are beyond the scope of this work, namely, our preference labels are assumed to follow the Bradley-Terry Model \citep{bradley1952rank}.


\paragraph{Notation. } The sets $\cS$ and $\cA$ are assumed to be countable throughout the paper.
For nonnegative sequences $\{x_n\}$ and $\{y_n\}$, we write $x_n = O(y_n)$ if $\limsup_{n\to\infty}{x_n}/{y_n} < \infty$, $y_n = \Omega(x_n)$ if $x_n = O(y_n)$, and $y_n = \Theta(x_n)$ if $x_n = O(y_n)$ and $x_n = \Omega(y_n)$. We further employ $\tilde{O}(\cdot), \tilde{\Omega}(\cdot)$, and $\tilde{\Theta}$ to hide $\polylog$ factors. 
For countable $\cX$ and $\cY$, we denote 
the family of probability kernels from $\cX$ to $\cY$ by $\Delta(\cY|\cX)$. 
For $g:\cX \to \RR$, its infinity norm is denoted by $\norm{g}_{\infty} \coloneqq \sup_{x\in\cX}|g(x)|$. 
For a pair of probability measures $P \ll Q$ on the same space and function $f: \RR_{+} \to \RR$, their $f$-divergence is $D_f(P\|Q) \coloneqq \int f(\ud P/ \ud Q) \ud Q$. Specifically, when $f(x)=x\log x$, $f$-divergence becomes KL divergence denoted as $\kl{P}{Q} \coloneqq \int \log({\ud P}/{\ud Q})\ud P$, and when $f(x) = |x-1|/2$, it becomes the total variation (TV) distance, which is denoted as $\tv{P}{Q} \coloneqq 0.5\int|\ud P - \ud Q|$. We use supp$(P)$ to denote the support set of $P$.


\section{KL-regularized Contextual Bandits}\label{sec:cb}

In this section, we introduce a pessimism-based algorithm, PCB-KL, for offline KL-regularized contextual bandits. We then showcase our novel analysis techniques for PCB-KL, which couples the algorithmic pessimism with the curvature property of KL-regularized objectives.

\subsection{Problem Setup}\label{sec:cb:setup}

We consider contextual bandit, which is denoted by a tuple $(\cS, \cA, r, \piref)$. Specifically, $\cS$ is the context space, $\cA$ is the action space and $r: \cS \times \cA \to [0, 1]$ is the reward function. In the offline setting, the agent only has access to an $\iid$ dataset $\cD = \{(s_i, a_i, r_i)\}_{i=1}^n$. Here $s_i's$ are states sampled from $\rho\in\Delta(\cS)$, $a_i \in \cA$ is the action taken from a \emph{behavior policy}, and $r_i$ is the observed reward given by $r_i = r(s_i, a_i) + \varepsilon_i$, where $\varepsilon_t$ is $1$-sub-Gaussian~\citep[Definition~5.2]{lattimore2020bandit}.
In this work, we consider the \emph{KL-regularized objective}
\begin{align}
    J_{\eta}(\pi) \coloneqq\EE_{(s,a) \sim \rho \times \pi} \bigg[r(s,a) - \eta^{-1} \log \frac{\pi(a|s)}{\piref(a|s)}\bigg], \label{eq:kl-objective-bandit}
\end{align}
where $\piref$ is a known reference policy and the ``inverse temperature'' $\eta$ controls the intensity of regularization. For simplicity, we assume that $\piref$ is also the behavior policy that generates the dataset $\cD$, which is similar to the type of ``behavior regularization'' studied in \citet{zhan2022offline}. The unique optimal policy $\pi^*_\eta \coloneqq \argmax_{\pi \in \Delta(\cA | \cS)} J_\eta(\pi)$ is
given by (See, e.g., \citealt[Proposition~7.16]{zhang2023ltbook})\footnote{We suppress $J_\eta$ into $J$ and $\pistar_\eta$ into $\pistar$ when they are clear in context in the following presentation.}
\begin{align}\label{eq:opt-exp}
    \pi^*(\cdot|s) \propto \piref(\cdot|s)\exp\big(\eta \cdot r(s, \cdot)\big), \forall s\in\cS.
\end{align}
A policy $\pi$ is said to be $\epsilon$-optimal if $\subopt_\mathrm{RKL}(\pi)\coloneqq J(\pi^*) - J(\pi) \leq \epsilon$ and the goal of the agent is to find one such policy using $\cD$. Note that $\subopt_\mathrm{RKL}(\cdot)$ is defined through \Cref{eq:kl-objective-bandit} and thus \textbf{depends on $\eta$}.
To ensure that $\epsilon$-optimality is achievable, we assume that $r$ lies in a known function class $\fcl \subset (\cS \times \cA \to [0,1])$, from which the agent obtains an estimator $\hat{r}$. 
More specifically, we work with general function approximation under realizability, which is as follows.
\begin{assumption}\label{assume:general-function-approx}
    For this known function class $\fcl \subset (\cS \times \cA \to [0,1])$, $\exists\fgt \in \fcl$ with $\fgt = r$. 
\end{assumption}
We also employ the standard notion of covering number \citep[Definition~5.1]{wainwright2019high} as the complexity measure of the reward function class
$\fcl$.

\begin{definition}[{$\epsilon$-net and covering number}]
Given a function class $\cG \subset (\cS \times \cA \to \RR)$, a finite set $\cG(\epsilon) \subset \cG$ is an $\epsilon$-net
of $\cG$ w.r.t. $\|\cdot\|_\infty$, if for any $g \in \cG$, there exists $g' \in \cG(\epsilon)$ such that $\| g - g'\|_\infty \leq \epsilon$. The $\epsilon$-covering number is the smallest cardinality $\cN_{\cG}(\epsilon)$ of such $\cG(\epsilon)$.
\end{definition}

\begin{assumption}\label{assume:poly-covering}
For any $\epsilon_c > 0$, the $\epsilon_c$-covering number $\cN_{\fcl}(\epsilon_c)$ of $\fcl$ is $\poly(\epsilon_c^{-1})$.
\end{assumption}
 \Cref{assume:poly-covering}
 allowing $\log \cN_{\cG}(\epsilon)$ to be roughly negligible is arguably mild. For example, when $\fcl$ is the class of linear functions of dimension $d$ and radius $R$, the covering number is $\cN_{\fcl}(\epsilon) = O((1+R\epsilon^{-1})^d)$ \citep[Lemma~D.6]{jin2020provably}, which satisfies \Cref{assume:poly-covering}.

\textbf{Concentrability.} The data quality of $\cD$ collected by $\piref$ is typically characterized by \emph{concentrability} in offline RL \citep{farahmand2010error,chen2019information,jiang2024offline}, which quantifies the ability of the behavioral policy to generate diverse actions.
We first define the density-ratio-based concentrability as follows. 

\begin{definition}[\emph{Density-ratio-based} concentrability]
For policy class $\Pi$, reference policy $\piref$, the density-ratio-based all-policy concentrability $C^{\Pi}$ is $C^{\Pi} \coloneqq \sup_{\pi \in \Pi, s \in \cS, a \in \cA}{\pi(a|s)} / {\piref(a|s)}$, whose single-policy counterpart under the optimal policy $\pistar$ is $C^{\pistar} \coloneqq \sup_{s\in \cS, a \in \cA} {\pistar(a|s)} / {\piref(a|s)}$.
\end{definition}

In the definition above, small all-policy concentrability intuitively corresponds to $\text{supp}(\piref)$ covering all possible inputs. On the other hand, small single-policy concentrability means that $\text{supp}(\piref)$ only subsumes $\text{supp}(\pi^*)$. In this paper, in addition to density-ratio-based concentrability, we also adopt the following $D^2$-based concentrabilites to better capturing the nature of function class $\fcl$. In detail, we start with the $D^2$-divergence as follows.

\begin{definition}\label{def:bandit:D-sq}
Given a function class $\fcl \subset ( \cS \times \cA \to \mathbb{R})$ and a fixed policy $\pi$, define the $D^2$-divergence $D^2_{\fcl}((s,a);\pi)$ as
\begin{align*}
    \sup_{ g,h \in \fcl} \frac{\big( g(s, a) - h(s, a)\big)^2}{\E_{(s',a') \sim \rho \times \pi}[( g(s', a') - h(s', a'))^2]}.
\end{align*}
\end{definition}
The ``eluder dimension''-type \Cref{def:bandit:D-sq} is directly inspired by \citet{di2024pessimistic,zhao2024sharp}, the intuition behind which is that given $(s,a) \in \cS \times \cA$, a small $D^2$-divergence indicates that for two functions $g$ and $h$, if they are close under the behavior policy $\pi$, then they will also be close on such pair $(s,a)$. Therefore, the $D^2$-divergence
quantifies how well the estimation on dataset collected by the behavior policy $\pi$ can be generalized to a specific state-action pair.

\begin{remark}
{For the tabular setting, a direct computation yields $D^2(s,a) = (\rho(s)\pi^{\text{ref}}(a|s))^{-1}$, which can be estimated by the visitation frequency empirically. Under linear function approximation, it is well known that $D^2(s,a) = \|\bphi(s,a)\|^2_{\bSigma^{-1}}$ under mild conditions of the parameter space, where $\bSigma = \mathbb{E}_{\rho\times \pi^{\text{ref}}} \bphi(s,a)\bphi(s,a)^\top$ is the covariance matrix, which can be estimated by empirical covariance matrices in practice, potentially with ridge regularization. For more general function classes like neural networks, the $D^2$ can also be efficiently approximated by heuristics as discussed in~\citet{xiong2024iterative,guptap3o,xu2025learning}.}
\end{remark}

We are now ready to define the two notions of concentrability conditions.

\begin{assumption}[All-policy concentrability]\label{assume:all-coverage-bandit}
    Given a reference policy $\piref$, there exists $D < \infty$ such that $D^2 = \sup_{(s,a) \in \cS \times \cA} D^2_{\fcl}((s,a); \piref) $.
\end{assumption}
\Cref{assume:all-coverage-bandit} indicates that the errors on any state-action pairs can be bounded by the error on the samples from $\rho \times \pi$ up to a factor $D$, whose relaxed counterpart under the same $\piref$ is as follows.
\begin{assumption}[Single-policy concentrability]\label{assume:single-coverage-bandit}
    $D_{\pi^*}^2 \coloneqq \E_{(s,a) \sim \rho \times \pi^*} D^2_{\fcl}((s,a); \piref) <\infty$.
\end{assumption}
\Cref{assume:single-coverage-bandit} indicates that the errors on the distributions of state-action pairs $\rho \times \pi^*$ can be bounded by the error on the samples from  $\rho \times \piref$ up to some constant. For both types, the single-policy concentrability assumption is strictly weaker than the all-policy concentrability assumption. However, in general, the two quantities characterizing single-policy concentrability $C^{\pi^*}$ and $D^2_{\pi^*}$ cannot be bounded by each other up to constant factors.
In particular, we have $D^2_{\pi^*} \leq |\cS||\cA|C^{\pi^*}$, indicating that $C^{\pi^*}$ subsumes $D^2_{\pi^*}$ when $|\cS|$ and $|\cA|$ can be seen as constants. \textcolor{black}{We refer the reader to Appendix~\ref{app:discussion-coverage} for a further discussion on the relation between $C^{\pi^*}$ and $D^2_{\pi^*}$.}

\subsection{Algorithm}

In this subsection, we present an offline bandit algorithm, $\algcb$, for KL-regularized contextual bandits in \Cref{algorithm:bandit-pess}. 
$\algcb$ first leverages least-square estimator to find a function $\fls \in \fcl$ that minimizes its risk on the offline dataset. In~\citet{zhao2024sharp}, such $\fls$ is directly applied to construct the estimated policy. 
In contrast, we construct a pessimistic estimator of $\fgt$ following the well-known pessimism principle in offline RL \citep{jin2021pessimism}.
Specifically, we define the bonus term $\Gamma_n$ through the confidence radius $\beta = \sqrt{128\log\big(2\cN_{\fcl}(\epsilon)/\delta\big)/3n + 18 \epsilon}$ as
\begin{align}
    \Gamma_n(s,a)= \beta  D_{\fcl}\big((s,a),\piref\big), \forall (s,a) \in \cS \times \cA. \label{eq:bonus-bandit}
\end{align}
We then obtain our pessimistic estimation $\fps$ by setting $\fps = \fls - \Gamma_n$, which is less than $\fgt$ with high probability. Formally, let the event $\confbandit(\delta)$ given $\delta>0$ defined as
\begin{align}\label{eq:confbandit}
    \confbandit(\delta) \coloneqq \Big\{ \sup\nolimits_{(s,a) \in \cS \times \cA} \Big[ \big| \fls - \fgt\big| - \Gamma_n \Big] (s,a) \leq 0 \Big\},
\end{align}
on which the least square estimation $\fls$ obtained in Line~\ref{line:bandit-lsq} of
\Cref{algorithm:bandit-pess}
does not deviate too much from the true function $\fgt$ and therefore $\fps$ is a pessimistic estimation of $\fgt$. We have the following lemma indicating that this event holds with high probability.
\begin{lemma}\label{lem:bandit:conf}
For all $\delta > 0$, $\confbandit(\delta)$ holds with probability at least $1-\delta$.
\end{lemma}

After obtaining the pessimistic estimation, $\algcb$ output the policy $\hat{\pi}$, which maximizes the estimated objective
\begin{align*}
    \hat{J}(\pi) =\EE_{(s,a) \sim \rho \times \pi} \bigg[\fps(s,a) - \eta^{-1} \log \frac{\pi(a|s)}{\piref(a|s)}\bigg],
\end{align*}
the maximizer of which is the counterpart of \Cref{eq:opt-exp}, i.e.,
\begin{align*}
    \hat \pi(a|s) \propto \piref(a|s) \exp \big(\eta \cdot \fps(s,a)\big).
\end{align*}

\begin{algorithm*}[t]
	\caption{Offline KL-Regularized Pessimistic Contextual Bandits (\algcb)}\label{algorithm:bandit-pess}
	\begin{algorithmic}[1]
    \REQUIRE regularization $\eta$, reference policy $\piref$, offline dataset $\cD$, function class $\fcl$

\STATE Least square estimation of reward function $\fls \in \argmin_{g \in \fcl} \sum_{(s_i,a_i, r_i) \in \cD} \big(g(s_i, a_i) - r_i\big)^2$ \label{line:bandit-lsq}
    \STATE Let $\fps \leftarrow \fls - \Gamma_n$, where $\Gamma_n$ is the bonus term in \Cref{eq:bonus-bandit} \label{line:bandit:pess}
    \ENSURE $\hat \pi(a|s) \propto \piref(a|s) \exp \big(\eta \cdot \fps(s,a)\big)$
\end{algorithmic}
\end{algorithm*}

\subsection{Theoretical Results}

The sample complexity for KL-regularized contextual bandits is settled in this subsection. We first give the upper bound of $\algcb$.

\begin{theorem}\label{thm:bandit}
Under \Cref{assume:single-coverage-bandit}, for sufficiently small $\epsilon \in (0, 1)$, if we set $\Gamma_n$ as in~\eqref{eq:bonus-bandit}, then $n = \tilde{O}\big(\eta D^2_{\pi^*} \epsilon^{-1} \log \cN_{\fcl}(\epsilon)\big)$ suffices to guarantee the output policy $\hat \pi$ of \Cref{algorithm:bandit-pess} to be $\epsilon$-optimal with probability at least $1 - \delta$.
\end{theorem}
Previously, \citet{zhao2024sharp} achieved an $\tilde{O}(\epsilon^{-1})$ sample complexity under \Cref{assume:all-coverage-bandit}. As a comparison, $\algcb$ achieves the same $\tilde{O}(\epsilon^{-1})$ sample complexity but only requiring \Cref{assume:single-coverage-bandit}, which is weaker than \Cref{assume:all-coverage-bandit}.
We also provide the sample complexity lower bound of KL-regularized contextual bandits in the following theorem, {which, together with~\Cref{thm:bandit}, demonstrates that single-policy concentrability is both necessary and sufficient for near-optimal offline learning evaluated by KL-regularized objectives.}
\begin{theorem}\label{thm:lowerbound-bandit}
For $\forall S \geq 1$, $\eta > 4\log 2$, $C^* \in (2, \exp(\eta/4)]$, and any algorithm \textsf{Alg}, there is a KL-regularized contextual bandit
with $C^{\pi^*} \leq C^*$ such that \textsf{Alg} requires
$\Omega\big( \min\{\eta\epsilon^{-1}, \epsilon^{-2}\} C^* \log \cN_{\fcl}(\epsilon)\big)$ samples to find an $\epsilon$-optimal policy for sufficiently small $\epsilon$.
\end{theorem}
Previously, \citet{zhao2024sharp} provided a sample complexity lower bound of $\Omega(\eta \log \cN_{\fcl}(\epsilon)/{\epsilon})$ under KL regularization. \citet{foster2025good} also provided a lower bound of $\Omega(C^{\pi^*})$ for KL-regularized objective to show the necessity of coverage. Compared to their results, our result shows that the \textit{multiplicative} dependency on $C^{\pi^*}$ is necessary for the first time.
\begin{remark}
\Cref{thm:lowerbound-bandit} shows that when $\epsilon$ is sufficiently small, any algorithm for offline KL-regularized contextual bandits requires at least $\Omega(\eta C^{\pi^*}) \epsilon^{-1} \log \cN_{\fcl}(\epsilon))$ samples to output an $\epsilon$-optimal policy.
The presence of $\exp(\poly(\eta))$ in the range of $C^*$ is inevitable, since we always have $C^{\pi^*} \leq \exp(\eta)$ in reverse KL regularized bandits with bounded rewards. 
\end{remark}

\begin{remark}
 As discussed before, we might have some easy instances with $D^2_{\pi^*} \leq C^{\pi^*}$, where $\algcb$ outperforms the lower bound. This does not volates Theorem~\ref{thm:lowerbound-bandit} since Theorem~\ref{thm:lowerbound-bandit} only guarantees that \textit{there exist} some hard instances that all algorithms require at least $\Omega( \min\{\eta\epsilon^{-1}, \epsilon^{-2}\} C^* \log \mathcal{N}_{G}(\epsilon))$ samples. 
\end{remark}



\subsection{Proof Overview of \Cref{thm:bandit}}

In this section, we summarize the novel techniques in the proof of \Cref{thm:bandit}, which is deferred to Appendix~\ref{app:proof-kl-upper}. At a high level, if we consider the regularized objective \Cref{eq:obj-teaser} multi-arm bandits, then $P\mapsto \kl{P}{Q}$ is $1$-strongly convex w.r.t. $\tv{\cdot}{\cdot}$ \citep[Exercise~I.37]{polyanskiy2025information}, and thus $J(\pi)$ is strongly concave. Therefore, $J(\pi^*) - J(\hat{\pi})$ is possible to be of the order $[\tv{\pi^*}{\hat{\pi}}]^2 \approx \tilde{O}(n^{-1})$, pretending that $\pi^*$ is the unconstrained maximizer. In detail, we follow the regret decomposition in~\citet{zhao2024sharp}, which is encompassed by the following lemma.

\begin{lemma}\label{lem:taylor-expansion}
Let $g:\cS \times \cA \to \RR$ be any reward function, then there exist some $\gamma \in [0,1]$ such that the sub-optimality gap \textcolor{black}{of $\pi_g(\cdot|s) \propto \piref(\cdot|s)\exp\big( \eta  g(s, \cdot)\big)$} can be bounded as
\begin{align*}
    J(\pi^*) - J(\pi_g) \leq \eta \EE_{(s,a) \sim \rho\times \pi_{\gamma}} \big[\big(\fgt - g\big)^2(s,a) \big],
\end{align*}
where \textcolor{black}{$g_{\gamma} \coloneqq \gamma g + (1-\gamma) g^*$ and $\pi_{\gamma}(\cdot|s) \propto \piref(\cdot|s)\exp\big( \eta  g_\gamma(s, \cdot)\big)$}.
\end{lemma}

In \citet{zhao2024sharp}, because the $g$ in \Cref{lem:taylor-expansion} is substituted with only the least-square estimator $\bar{g}$ with no extra structures, the reliance on the ``mid-point'' policy $\pi_{\gamma}$ can only be controlled all-policy concentrability. However, our $g$ is the pessimistic estimator $\hat{g}$ of $\fgt$ in \Cref{algorithm:bandit-pess}, and thus the presence of $\pi_\gamma$ can be eliminated for free: let $G(\gamma) \coloneqq \EE_{\rho \times \pi_{\gamma}}\Big[ \big(\fps - {\fgt}\big)^2(s,a) \Big]$ and $\triangle(s,a) \coloneqq \big(\fps - {\fgt})(s,a) \leq 0$, then a direct computation (detailed in the proof of \Cref{lem:expansion-to-endpoint}) yields
\begin{align}
    G'(\gamma) = \eta \EE_{\rho}\Big[ \EE_{\pi_{\gamma}} \big[ \triangle^3(s,a) \big] - \EE_{\pi_{\gamma}}\big[ \triangle^2(s,a) \big]\EE_{\pi_{\gamma}}\big[ \triangle(s,a) \big] \Big]  \leq 0. \label{eq:key-machine}
\end{align}
This gives $J(\pi^*) - J(\hat{\pi}) \leq \eta \EE_{\rho \times \pi^*}\big[ (\fps - {\fgt})^2(s,a) \big]$, which can be bounded with single-policy concentrability while still achieves the sharp dependency $\epsilon^{-1}$ on $\epsilon$. 
Here,
\eqref{eq:key-machine}
holds due to a moment-based machinery in \Cref{lem:non-positive-cov}.
\begin{lemma}\label{lem:non-positive-cov}
    If $\PP(X \leq 0) = 1$
    and $\EE|X|^3 < \infty$, then $\EE[X^3] - \EE[X^2]\EE[X] \leq 0$.
\end{lemma}
The intuition behind \Cref{lem:non-positive-cov} is natural: $X$ and $X^2$ cannot be positively correlated. Moreover, to the best of our knowledge, we are the first to unveil this moment-based structure
in our
non-standard pessimism-based
analysis,
from which the sharp upper bound
follows. While pessimism is widely adopted to derive near-optimal statistical rates under
single-policy concentrability in offline RL with reward maximization as the goal (See, e.g., \citet{jin2021pessimism,xiong2022nearly}),
the standard pessimism-based pipeline is not sharp enough for bounding the $\subopt_\mathrm{RKL}(\hat{\pi})$ \emph{defined through regularized objectives}, the
reason
of which is detailed in the last paragraph of \Cref{sec:attempts}. 


\section{$f$-divergence-regularized Contextual Bandits}\label{sec:cb-f}

As discussed in \Cref{sec:cb}, the fast rate implied by \Cref{thm:bandit,thm:lowerbound-bandit} is primarily achieved due to the strong convexity of $\pi \mapsto \KL(\pi \| \piref)$. However, KL is just an instance of $f$-divergence with $f(x) = x\log x$, which is only locally strongly convex but not strongly convex. Motivated by this observation, we further examine
$f$-divergence regularization with strongly convex $f$, which may introduce a more favorable curvature in the performance metric of offline learning in principle.

\subsection{Problem Setup}

We study a contextual bandit setting similar to that in \Cref{sec:cb:setup}. In this section, we consider the following $f$-divergence regularized objective
\begin{align}
    J_{\eta, D_f}(\pi) \coloneqq\EE_{(s,a) \sim \rho \times \pi}[r(s,a)] - \eta^{-1}\EE_{s \sim \rho}\big[D_{f}\big(\pi(\cdot|s) \| \piref(\cdot|s)\big)\big], \label{eq:f-objective-bandit}
\end{align}
where $\eta$ is the
 regularization intensity and $D_{f}(p \| q) \coloneqq \EE_{a \sim q}\Big[f\big({p(a)}/{q(a)}\big)\Big]$ is the $f$-divergence.
Let the optimal policy be $\pistar_{\eta, D_f} \coloneqq \argmax_{\pi \in \Delta(\cA | \cS)} J_{\eta, D_f}(\pi)$ and we re-define the learning objective as searching for a policy $\pi$ with $\suboptfdiv(\pi) \coloneqq J(\pistar) - J(\pi) \leq \epsilon$.\footnote{We again suppress $J_{\eta, D_f}(\cdot)$ into $J(\cdot)$ and $\pistar_{\eta, D_f}$ into $\pistar$ when there is no confusion.}
We consider those functions $f :(0,+\infty) \to \RR$ with a nice positive curvature condition in \Cref{assume:f-strong-convex}.
\begin{assumption}\label{assume:f-strong-convex}
 $f$ is $\alpha$-strongly convex, twice continuously differentiable, and $f(1) = 0$.
\end{assumption}
Many elementary functions like quadratic polynomials naturally satisfy \Cref{assume:f-strong-convex}.
For instance, the 1-strongly convex $f(x) = (x-1)^2/2$ induces $D_f(P\|Q) = \chi^{2}(P\|Q)$, which is the $\chi^2$-divergence recently considered in RL literature (see e.g., \citet{zhan2022offline,huang2024correcting,amortila2024scalable}). This regularization exhibits a promising theoretical potential for relaxing the data coverage requirement for efficient offline policy learning~\citep{huang2024correcting} and to be effective in preventing reward hacking~\citep{laidlawcorrelated} against unregularized objectives. These favorable benefits are primary due to the observation that strongly convex $f$'s impose a stronger penalization on actions out of the coverage of $\piref$.


\subsection{Algorithm and Main Results}

\begin{algorithm*}[t]
	\caption{Offline $f$-divergence Regularized Contextual Bandits (\algfcb)}\label{algorithm:bandit-fdiv}
	\begin{algorithmic}[1]
    \REQUIRE regularization $\eta$, reference policy $\piref$, function class $\fcl$, offline dataset $\cD$

\STATE Least square estimation $\fls \in \argmin_{g \in \fcl} \sum_{(s_i,a_i, r_i) \in \cD} \big(g(s_i, a_i) - r_i\big)^2$
\STATE Compute the optimal policy under the least-square reward estimator $\fls$ for $s \in \cS$ as

$$\hat \pi(\cdot | s) \leftarrow \argmax_{\pi(\cdot | s) \in \Delta(\cA)} \dotp{\pi(\cdot|s)}{\fls(s, \cdot)} + \eta^{-1} D_f\big(\pi (\cdot|s) \| \piref(\cdot|s)\big) $$
    \ENSURE $\hat \pi$
\end{algorithmic}
\end{algorithm*}

In this subsection, we present an offline learning algorithm for $f$-divergence regularized bandit, $\algfcb$, in \Cref{algorithm:bandit-fdiv}.
\Cref{algorithm:bandit-fdiv} first leverages least-square estimator to find a function $\fls \in \fcl$ that minimizes its risk on the offline dataset. The algorithm then uses the least squares estimation $\fls$ to construct the output policy $\hat{\pi}$. Compared to \Cref{algorithm:bandit-pess}, $\algfcb$ does not require any procedure to construct pessimistic reward estimation, whose sample complexity upper bound is given as follows.

\begin{theorem}\label{thm:f-upper-bound}
Under \Cref{assume:f-strong-convex}, for sufficiently small $\epsilon \in (0, 1)$, with probability at least $1 - \delta$, $n = \tilde{O}(\alpha^{-1} \eta \epsilon^{-1} \log \cN_{\fcl}(\epsilon))$ is sufficient to guarantee the output policy $\hat \pi$ of $\algfcb$ to be $\epsilon$-optimal.
\end{theorem}

\begin{remark}
Compared to the $D^2_{\pi^*}$ dependency in \Cref{thm:bandit}, \Cref{thm:f-upper-bound} shows that the sample complexity of \Cref{algorithm:bandit-fdiv} gets rid of the dependency on any data coverage conditions when $f$ is strongly convex.
Intuitively, this is because the $f$-divergence
regularization in this case
is much stronger, so that both $\pi^*$ and $\hat \pi$ are close enough to $\piref$.
\end{remark}

The following hardness result justify the near-optimality of \Cref{thm:f-upper-bound} for $f$-divergence-regularized contextual bandits.
\begin{theorem}\label{thm:f-lower-bound}
For any $\epsilon \in (0, 1), \alpha > 0,\eta > 0, S > 32/3\cdot\log2$, sufficiently small $\epsilon$, and algorithm \textsf{Alg}, there is an $\alpha$-strongly-convex function $f$ and an $f$-divergence-regularized contextual bandit instance such that \textsf{Alg} requires at least $\Omega\big(\alpha^{-1}\eta \epsilon^{-1} \log \cN_{\fcl}(\epsilon) \big)$ samples to return an $\epsilon$-optimal policy.
\end{theorem}


\subsection{Proof Overview of \Cref{thm:f-upper-bound}}

We provide an overview of key analysis techniques for proving \Cref{thm:f-upper-bound}. Unlike KL-regularization, the $\pi^*$ under $f$-divergence might not have a closed form. This means that the proof of Lemma~\ref{lem:taylor-expansion}, which relies on the closed form of $\pi^*$, cannot be directly adopted. Therefore, we address this from a dual-Bregman perspective, \textcolor{black}{inspired by~\citet{abernethy2015fighting}}. For the simplicity of presentation, we consider multi-armed bandits here and omit the subscript for context $s$. 

We consider the function $H(\pi) = \eta^{-1}D_f(\pi \| \piref)$, which is the regularizer in the objective. Then its convex conjugate is given by $H^*(r) = \sup_{\pi \in \Delta^d} \{ \la \pi, r \ra - H(\pi)\}$, which is exactly the expected reward obtained by the optimal policy given reward function $r$. 
One observation is that when $f$ is strongly convex, the induced $f$-divergence, and therefore the function $H$ are also strongly convex.
Therefore, let $\pi_r = \argmax_{\pi}\{ \la \pi, r \ra - H_s(\pi)\}$ given some reward function $r$, the strong convexity of $H(\pi)$ gives that $\nabla H^*(r) = \pi_r$. This leads to the following regret decomposition, which is one of our key observations:
\begin{align*}
    J(\pi^*) - J(\hat \pi) &= \EE_{a \sim \pi^*}[\fgt(a)] - \EE_{a \sim \hat \pi}[\fgt(a)] - \eta^{-1}\big[D_{f}(\pi^* \| \piref) - D_{f}(\hat \pi \| \piref) \big]  \\
    & = H^*(\fgt) - H^*(\fls) - \la \hat \pi, \fgt - \fls \ra \\
    & = H^*(\fgt) - H^*(\fls) - \la \nabla H^*(\fls), \fgt - \fls \ra,
\end{align*}
which is the Bregman divergence of the dual function $H^*$ and therefore can be bounded by $(\fgt -\fls)^\top \nabla^2H^*(\tilde g)(\fgt -\fls)$ for some $\tilde g$. By Proposition 3.2 in~\citet{penot1994sub}, when $H$ is strongly convex, we can bound $\nabla^2 H^*(\tilde g)$ as follows
\begin{align*}
    \nabla^2 H^*(\tilde g) & \preceq \big(\nabla^2H(\pi_{\tilde g})\big)^{-1} \preceq \alpha^{-1} \eta  \mathrm{diag}\big(\piref(a_1), \cdots,\piref(a_{|\cA|})\big),
\end{align*}
which enables us to bound $(\fgt -\fls)^\top \nabla^2 H^*(\tilde g)(\fgt -\fls)$ by $\alpha^{-1}\eta \EE_{\piref}[(\fgt - \fps)^2]$. Since $\EE_{\piref}[(\fgt - \fps)^2]$ is not related to $\pi^*$, the upper bound is independent of any notion of concentrability.

\begin{figure}[htb!]
\subfigure[Empirical rate for KL.]{\label{subfig:kl}\includegraphics[width=0.3\textwidth]{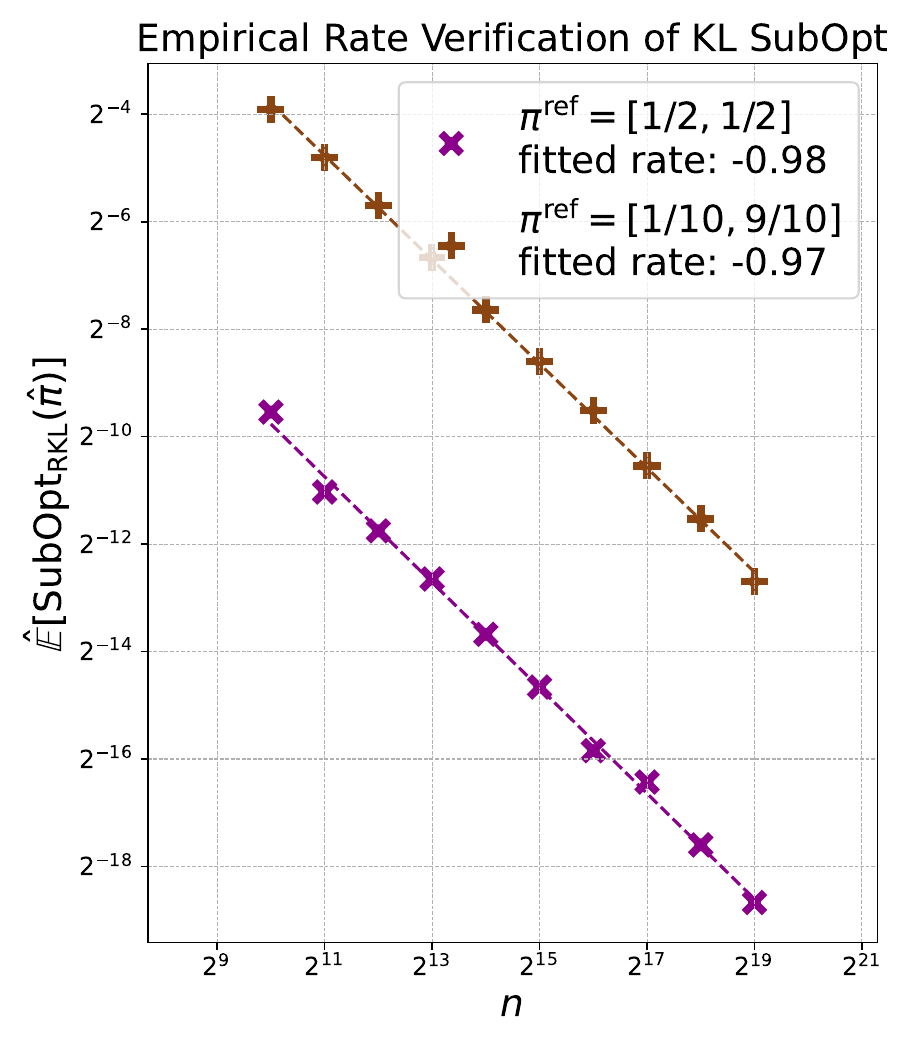}}
\subfigure[LHS: Empirical rate for $\chi^2$; RHS: Empirical scaling of $\alpha$.]{\label{subfig:f}\includegraphics[width=0.7\textwidth]{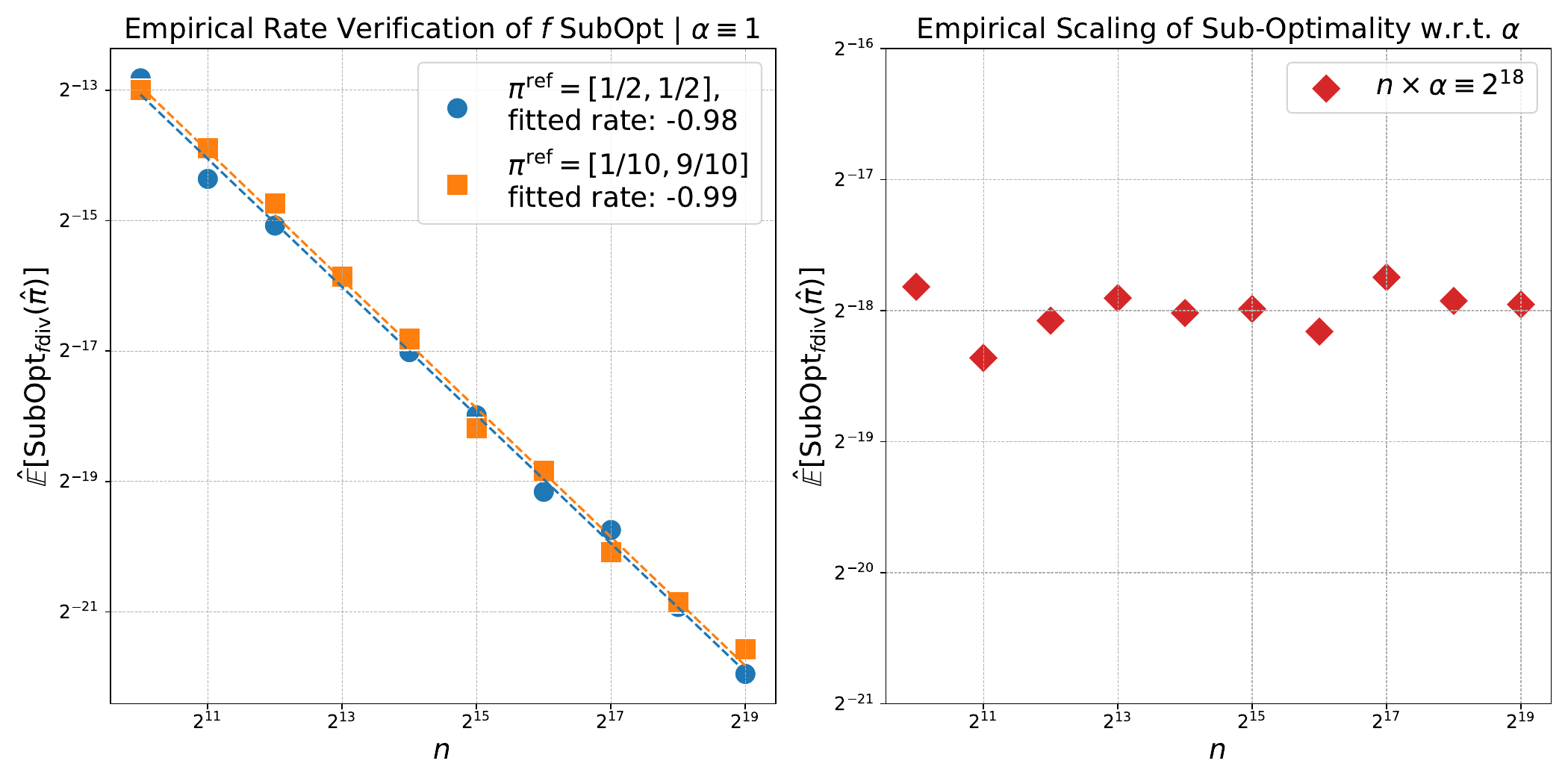}}
\caption{The empirical relation between $\log_2 n$ and $\log_2 \subopt$. The \emph{fitted rate} means the slope of $\log_2 n \sim \log_2 \subopt$ estimated via linear regression. Here $n$ is the sample size. Every point is the {average} over \textbf{100} independent trials.}
\label{fig:veri}
\end{figure}

\section{Experiments}\label{sec:numerical}


\textbf{Simulation on multi-armed bandits.} We first empirically check the correctness of our matching bounds for KL and $f$-divergence on the simplest testbed: \emph{two-armed} bandits, i.e., $\cA = \{0, 1\}$. We use one hard instance constructed in the proof of \Cref{thm:lowerbound-bandit} (\Cref{sec:lowerbound-bandit}) for the simulation under KL and one hard instance constructed in the proof of \Cref{thm:f-lower-bound} (\Cref{sec:f-lower-bound}) for the simulation under $f$-divergence with $f(x) = \alpha(x-1)^2/2$.
Recall that the dependency on $\epsilon$ in all sample complexity bounds above is $\tilde{\Theta}(\epsilon^{-1})$, and thus both $\subopt_\mathrm{RKL}$ and $\suboptfdiv$ should be roughly proportional to $n^{-1}$ as a function of the sample size $n$, which can be verified from the linear regression between $\log_2 n$ and $\log_2 \subopt$; i.e., the estimated slope should be approximately $-1$. Therefore, the two fitted rates in \Cref{subfig:kl} indicates that \algcb{} indeed achieves the near-optimal statistical rate $n^{-1}$ under different $\piref$'s and the counterparts in the LHS of \Cref{subfig:f} indicates the near-optimality of \algfcb{} empirically. The contrast between \Cref{subfig:kl} and the LHS of \Cref{subfig:f} also corroborates that the sample complexity against the KL-regularized objective positively depend on the concentrability, while that against the $\chi^2$-divergence-regularized objective does not vary with the coverage condition of $\piref$. Moreover, on top of the hard instance for $f$-divergence, we further set $\alpha = 2^{15} / n$ to numerically examine the scaling of $\suboptfdiv$ w.r.t. the strong convexity modulus $\alpha$. As shown on the RHS of \Cref{subfig:f}, $\suboptfdiv$ remains stable as $n$ goes up given $n \alpha \equiv 2^{15}$; therefore, \Cref{subfig:f} also empirically verified that $\suboptfdiv$ is inversely proportional to $\alpha$.



{
\color{black}


\begin{figure}[ht]
\begin{center}
    \subfigure[Empirical rate for KL.]{\label{subfig:kl-linear}\includegraphics[width=0.35\textwidth]{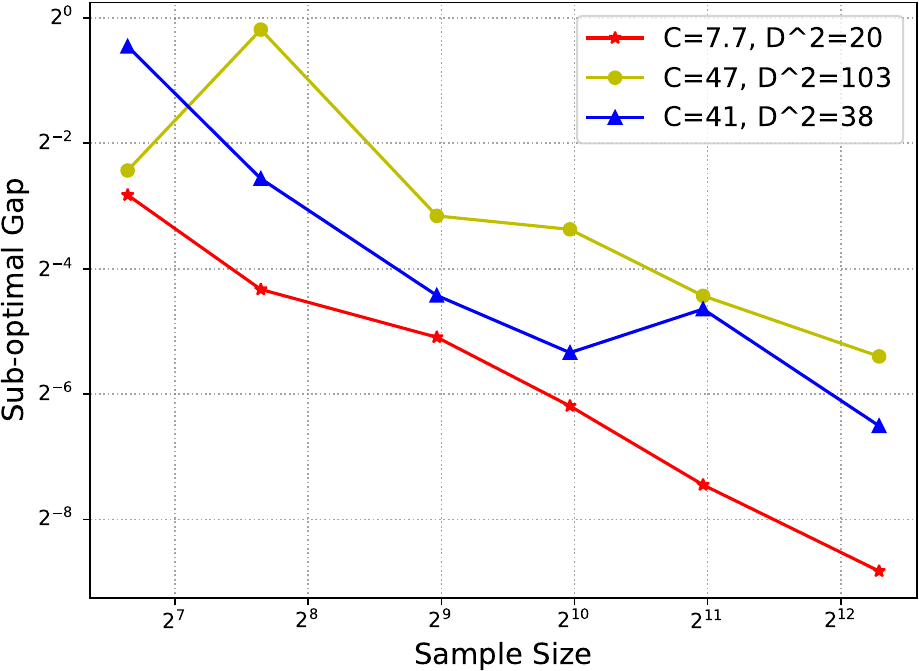}}
\subfigure[Empirical rate for $\chi^2$]{\label{subfig:chi2-linear}\includegraphics[width=0.35\textwidth]{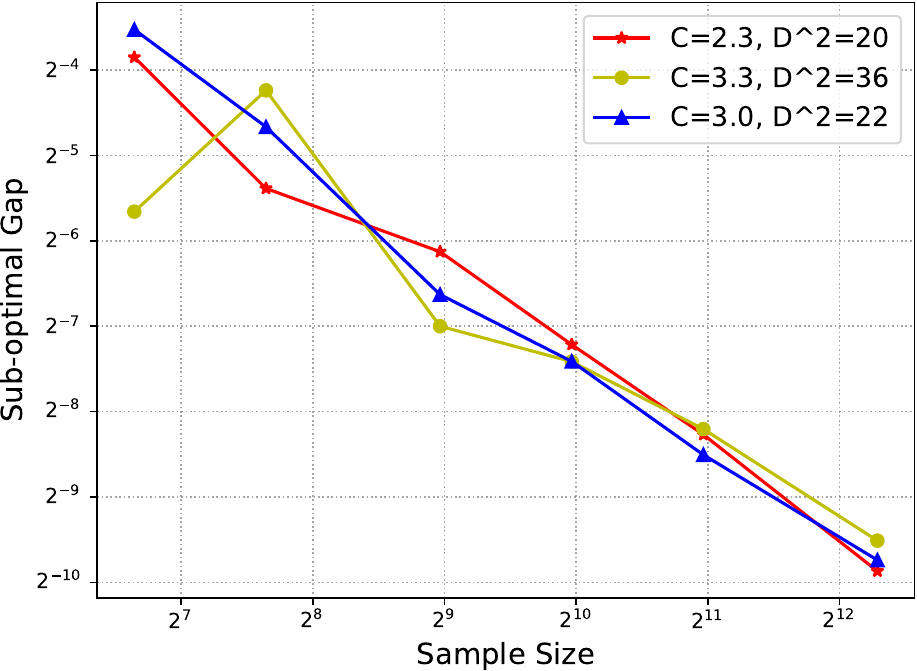}}
\end{center}
\caption{The empirical relation between $\log_2 n$ and $\log_2 \subopt$ for linear bandits. In the legend, we denote $C^{\pi^*}$ (resp. $D^2_{\pistar}$) by \texttt{C} (resp. \texttt{D\^{}2}).}
\label{fig:linear}
\end{figure}

\textbf{Simulation on linear bandits.} We then simulate a linear bandit as follows. The constructions of the feature map $\bphi$, ground-truth parameter $\btheta^*$ and the induced reward are detailed in \Cref{app:lin-exp-details}. The behavior policy is constructed as $\piref = \beta \mathsf{Unif}(\cA) + (1-\beta) \mathsf{Unif}(\cA_k)$, where $\cA_k \subset \cA$ be the subset such that $\cA_k$ consists of the $k$ arms with the lowest expected reward. We consider three different behavior policies, $(\beta, k) \in \{(1, \cdot), (0.1, 4), (0.05, 20)\}$, which induces various $C^{\pi^*}$ and $D^2_{\pi^*}$ so as to demonstrate the influence of coverage under different regularization. The results are compiled in \Cref{fig:linear}. Specifically, for the KL-regularized cases depicted in Figure~\ref{subfig:kl-linear}, we see that as the coverage coefficients $C^{\pi^*}$ and $D^2_{\pistar}$ vary, there is a consistent sub-optimality gap margin between these instances. In contrast, \Cref{subfig:chi2-linear}shows that the sub-optimality gaps under different instances (with distinct coverage coefficient) are very close for sufficiently large sample sizes. These results corroborate our theoretical finding that the sample complexity w.r.t. KL-regularized objectives is concentrability-dependent but that w.r.t $f$-divergence ones is not (for strongly convex $f$).


\begin{figure}[ht]
\begin{center}
    \subfigure[Empirical rate for KL]{\label{subfig:kl-mnist}\includegraphics[width=0.35\textwidth]{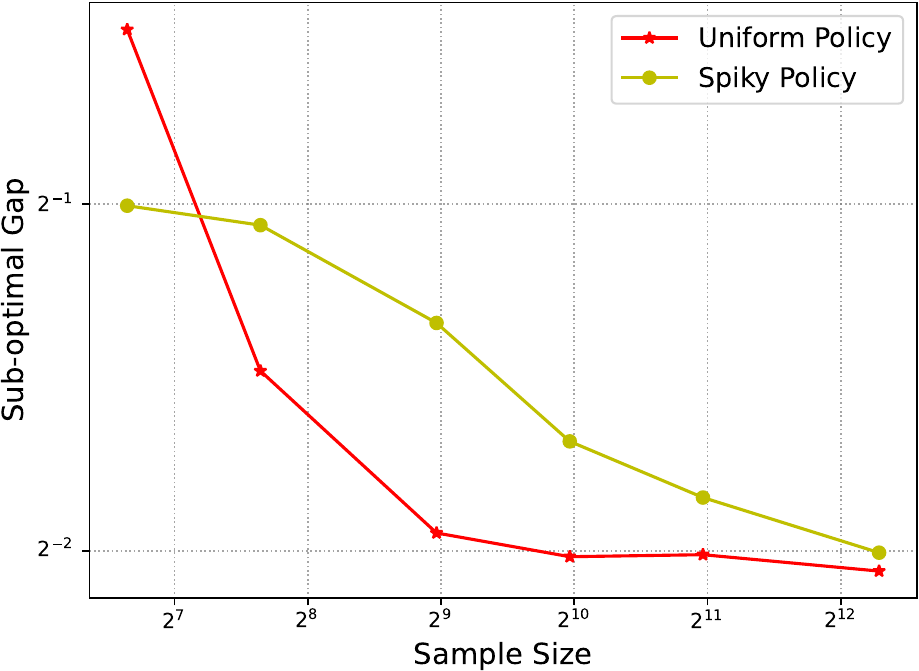}}
\subfigure[Empirical rate for $\chi^2$]{\label{subfig:chi2-mnist}\includegraphics[width=0.35\textwidth]{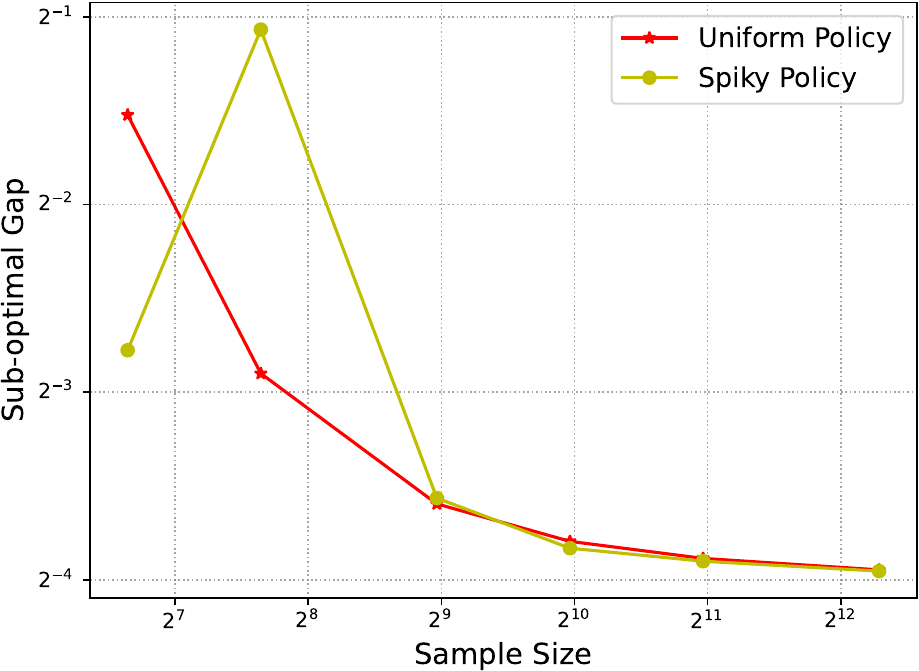}}
\end{center}
\caption{The empirical relation between $\log_2 n$ and $\log_2 \subopt$ on MNIST dataset. }
\label{fig:mnist}
\end{figure}

\textbf{Real-world experiments.} We further verify our theory on a vision dataset, MNIST~\citep{lecun1998mnist}. The construction of the feature map 
is detailed in \Cref{app:exp-details}. We consider two reference polices, a uniform policy $\mathsf{Unif}(\cA)$ and a spiky policy $0.5\mathsf{Unif}(\cA) + 0.5\mathsf{Dirac}(\{0\})$ to obtain instances with different concentrability coefficients. Figure~\ref{fig:mnist} exhibits the $\subopt$ curves, which show that under KL-regularization, when sample size is not large enough, there exists a considerable gap between instances with different behavior policy, but the gap is vanishing as the sample size increases. On the other hand, as for $\chi^2$-divergence regularization, such a gap vanishes quickly when the sample size becomes moderately large and the sub-optimal gap remains similar for larger sample sizes. These results are consistent with the simulation in \Cref{sec:numerical} and our theoretical findings.

}

\section{Conclusion and Future Work}

In this work, we take the first step towards fully understanding the statistical efficiency \emph{with respect to $f$-divergence-regularized objectives} of offline policy learning by sharp analyses for two empirically relevant subclasses. (1) We are the first to show that single-policy concentrability is nearly the right coverage condition for
reverse KL to achieve the fast $\tilde{\Theta}(\epsilon^{-1})$ sample complexity. The novel techniques in algorithm analysis
leverages the curvature of KL-regularized objectives and integrates pessimism with a newly identified moment-based observation,
enabling a neat refinement of a mean-value-type argument to the extreme; which are decoupled from tricky algorithmic tweaks, and thus might be of independent interest.
(2) If strong convexity is further imposed on $f$, our fast $\tilde{\Theta}(\epsilon^{-1})$ sample complexity is provably free of any coverage dependency. Unlike those for KL, the upper bound arguments for strongly convex $f$ do not rely on specific closed-form solutions of the regularized objective maximizer.

All techniques in this work can be generalized beyond vanilla absolute reward feedback, as certified by CDBs, which is detailed in \Cref{app:cdb} under a slightly different notion of $D^2$ tailored for pairwise comparison feedback. However, for reverse-KL regularization, the $D^2_{\pistar}$ in the upper bound and the $C^{\pistar}$ in the lower bound still does not perfectly match. Also, for general $f$-divergence other than reverse-KL, our analyses require $f$ to be twice-continuously differentiable and strongly convex. Fully closing the gap under reverse-KL regularization and extending the analysis to general $f$-divergences are interesting directions for future work.



\section*{The Use of Large Language Models (LLMs)}
We use LLMs as a tool to refine our writing and correct grammatical errors.

\section*{Acknowledgment}
We thank the anonymous reviewers and area chair for their helpful comments. QZ, KJ, HZ and QG are supported in part by the National Science Foundation DMS-2323113 and IIS-2403400. HZ is also supported in part by Amazon PhD Fellowship.
The views and conclusions contained in this paper are those of the authors and should not be interpreted as representing any funding agencies.

\bibliography{reference}
\bibliographystyle{iclr2026_conference}

\appendix

\tableofcontents

\section{Additional Review of Existing Results}\label{sec:exist-result}


\paragraph{Additional notations.} 
Besides the notation introduced in \Cref{sec:intro}, we will use the following notations in Appendix. We denote $[N] \coloneqq \{1, \cdots, N\}$ for any positive integer $N$. Boldfaced lower case (resp. upper case) letters are reserved for vectors (resp. matrices).
Given a positive definite $\bSigma \in \RR^{d\times d}$ and $\xb\in \RR^d$, we denote the vector's Euclidean norm by $\|\xb\|_2$ and define $\|\xb\|_{\bSigma}=\sqrt{\xb^\top\bSigma\xb}$.
We use $\mathsf{Bern}(p)$ to denote Bernoulli distribution with expectation $p$ and $\mathsf{Unif}(\cX)$ for the uniform distribution on finite set $\cX$. For $x \in \RR^{|\cA|}$, we denote $\|x\|_1 = \sum_{a \in \cA} |x_a|$. We also denote $x_n = \Omega(y_n)$ by $x_n \gtrsim y_n$ in Appendix. We use $d_H$ for Hamming distance.


\subsection{Previous Attempts on Understanding KL-regularized RL}\label{sec:attempts}

There has been a surge of interest in understanding the principle behind KL-regularized RL. \citet{ahmed2019understanding,liu2019regularization} studied by ablation the effect of entropy regularization on the stability of policy improvement in policy optimization, the regret of which has been
rigorously settled under the classic online mirror descent framework \citep{cai2020provably,he2022near,ji2023horizon}. \citet{neu2017unified}
unified popular KL-regularized policy optimization algorithms under a convex optimization framework, but the interplay with the data was left
untouched. A series of work \citep{geist2019theory,vieillard2020leverage,kozuno2022kl} then analyzed the sample complexity of algorithms using KL/entropy-type proximal terms
with respect to the previous iteration or/and entropy regularizer 
with improved dependence on the effective horizon in
discounted Markov decision processes.
However, the performance metric
in 
these studies is still
the unregularized reward maximization objective, under which the
sample complexity
for finding an $\epsilon$-optimal policy is at least equal to the statistical limit $\Omega(\epsilon^{-2})$.

\paragraph{Convergence under regularized objectives. } Several recent studies \citep{xie2024exploratory,xiong2024iterative,zhao2024sharp,zhao2025logarithmic,foster2025good} switched the focus to analyzing the sub-optimality
guarantee
with respect to 
the regularized objective \eqref{eq:obj-teaser}. In particular, \citet{xie2024exploratory} studied token-level Markov decision processes (MDPs) and proposed a KL-regularized RL algorithm named XPO, which achieves $\tilde{O}(\epsilon^{-2})$ sample complexity under their notion of all-policy concentrability. \citet{xiong2024iterative} proposed an Offline GSHF algorithm via the principle of \emph{pessimism in the face of uncertainty}, and proved $\tilde{O}(\epsilon^{-2})$ sample complexity under single-policy concentrability (See \Cref{sec:cb:setup} for detailed definitions of concentrability). On the other hand, the sharp analysis in \citet{zhao2024sharp} yields the optimal sample complexity $\tilde{O}(\epsilon^{-1})$, but requires all-policy concentrability \citep[Definition~2.6]{zhao2024sharp}, i.e., the behavior policy $\piref$ is required to cover the entire function class for all possible policies. \citet{zhao2025logarithmic} considered the online episodic MDP setting, which inherently does not need any notion of data coverage and thus their results are not directly adaptable to our offline setting. 
 \citet{foster2025good} considered an interesting hybrid setting in which the $n$ state-action pairs are still from the offline dataset but $\Omega(n)$ online reward queries and policy switches are allowed; in contrast, in our setting, all reward signals are obtained in a purely offline manner.


\paragraph{Previous analyses and results in detail.} Here, we briefly discuss the direct adaptation of previous sample complexity analysis and results (with respect to KL-regularized objectives) to our setting and demonstrate the reason why
theirs cannot imply an $\tilde{O}(\epsilon^{-1})$ sample complexity without all-policy concentrability. In previous analysis of pessimism for unregularized objectives~\citep{jin2021pessimism,xiong2022nearly}, the sub-optimality gap is decomposed via the performance difference lemma as follows
\begin{align*}
    J(\pi^*) - J(\hat \pi) &= \EE_{a \sim \pi^*}[\fgt(a)] - \EE_{a \sim \hat \pi}[\fgt(a)] - \eta^{-1}\KL(\pi^* \| \piref) + \eta^{-1}\KL(\hat \pi \| \piref)\\
    & \leq  \EE_{a \sim \pi^*}[\fgt(a)] - \EE_{a \sim \hat \pi}[\fps(a)] - \eta^{-1}\KL(\pi^* \| \piref) + \eta^{-1}\KL(\hat \pi \| \piref) \\
    & \leq \EE_{a \sim \pi^*}[\fgt(a)] - \EE_{a \sim \pi^*}[\fps(a)] - \eta^{-1}\KL(\pi^* \| \piref) + \eta^{-1}\KL(\pi^* \| \piref) \\
    & = \EE_{a \sim \pi^*}[\fgt(a) - \fps(a)],
\end{align*}
where the first inequality holds due to pessimism and last inequality holds due to $\hat \pi$ is optimal for $\fps$. Notably, the KL-regularization term is canceled out in the analysis, leading to a loose sample complexity $\tilde{O}(\epsilon^{-2})$ since the curvature of KL-divergence is not exploited. Specifically, under linear function approximation, this performance gap, obtained by \citet{xiong2024iterative} becomes
\begin{align*}
    J(\pi^*) - J(\pi) \leq \big\| \EE_{\rho \times \pi^*}[\bphi(s,a)] - \bnu \big\|_{\bSigma_{\text{off}}^{-1}} \eqqcolon \text{RHS},
\end{align*}
where $\bnu$ is the reference vector, $\bphi(s,a) \in \RR^d$ is the feature map, and $\Sigma_{\text{off}} = \sum_{i=1}^n \bphi(s_i, a_i)\bphi(s_i, a_i)^\top$ is the sample covariance matrix. However, we can show that $\text{RHS}$ can be bounded from \emph{below} by
\begin{align*}
    \big\| \EE_{(s,a) \sim \rho \times \pi^*}[\bphi(s,a)] - \bnu \big\| \sqrt{\lambda_{\min}(\bSigma_{\text{off}}^{-1})} 
    & = \big\| \EE_{(s,a) \sim \rho \times \pi^*}[\bphi(s,a)] - \bnu \big\| \lambda_{\max}(\bSigma_{\text{off}})^{-1/2}  \\
    & \geq \big\| \EE_{(s,a) \sim \rho \times \pi^*}[\bphi(s,a)] - \bnu \big\| \text{tr}(\bSigma_{\text{off}})^{-1/2} \\
    & = \big\| \EE_{(s,a) \sim \rho \times \pi^*}[\bphi(s,a)] - \bnu \big\| \bigg(\sum_{i=1}^n \|\bphi(s_i, a_i)\|_2^2 \bigg)^{-1/2} \\
& = \Omega(n^{-1/2}),
\end{align*}
where $\lambda_{\min}$ and $\lambda_{\max}$ is the minimum and maximum eigenvalue of a matrix, the first inequality holds due to the fact that $\xb^\top \bSigma \xb \geq \|\xb\|_2^2 \lambda_{\min}(\bSigma)$ and the second inequality holds due to $\lambda_{\max}(\bSigma) \leq \text{tr}(\bSigma)$.
\citet{zhao2024sharp} proposed a two-stage learning algorithm and obtained an $\tilde{O}(\epsilon^{-1})$ sample complexity for online KL-regularized bandits. The algorithm can be adopted to offline learning by removing the second stage\footnote{This can be done by setting the $n$ in their paper to $0$.} and treat the samples from first stage as the offline dataset. An analogous analysis gives a sample complexity of $\tilde{O}(D^2\epsilon^{-1})$, where $D^2$ is the all-policy concentrability.

{
\color{black}

\section{Additional Discussion of Relation between Coverage Measures}\label{app:discussion-coverage}

In this section, we provide more illustrations on the relation between two coverage measures, $D^2_{\pistar}$ and $C^{\pistar}$. In particular, we provide two cases under linear function approximation, on one of which $D^2_{\pistar} = \Theta(d C^{\pi^*})$ and on the other we have $D^2_{\pistar} \ll C^{\pi^*}$, where $d$ is the dimension of the function class. We summarized them as two propositions.

\begin{proposition}
There exist a KL-regularized linear bandit instance, such that $D^2_{\pistar} = \Theta(d C^{\pi^*})$.
\end{proposition}

\begin{proof}
We construct the instance as follows. Let $d = 2A+1$ be some odd number and consider an $2A+1$-armed bandit, such that the feature vector of the $i$-th arm, $\bphi(a_i) = \eb_i \in \RR^d$, which has $1$ on its $i$-th entry and $0$ on all other entries. The reference policy $\piref(a_i) = (2AC)^{-1}$ for $i \in [2A]$ and $\piref(a_{2A+1}) = (C-1)/C$, where $2C-1 = e^{\eta}$. The ground truth reward function $\btheta^* = \sum_{i \leq A}\eb_i$ and the function class is given by all $\|\btheta\|_{\infty} \leq 1$. By construction, we know that $\pi^*(a_i) \geq \piref(a_i)$ if and only if $i \in [A]$ and its closed form is given by
\begin{align*}
    \pi^*(a_i) = \frac{1}{A}\frac{e^\eta}{e^{\eta} + 2C-1} = \frac{1}{2A},
\end{align*}
which gives $C^{\pi^*} = C$. Now we compute the $D^2_{\pi^*}$ of this instance. For all $i \in [A]$, we know that
\begin{align*}
    D^2(a_i) = \sup_{\|\btheta\|_{\infty} \leq 2} \frac{\la \btheta, \eb_i \ra^2}{\EE_{\piref}\la \btheta, \eb_j \ra^2} = 2CA = \Theta(Cd),
\end{align*}
where the second equation holds with $\btheta = \eb_i$. Taking expectation over $\pistar$, we have
\begin{align*}
    D^2_{\pistar} \geq \sum_{i \in [A]}D^2(a_i) = \Theta(C^{\pistar}d),
\end{align*}
which concludes the proof.
\end{proof}

The following proposition provides another instance on which $D^2_{\pistar} \ll C^{\pistar}$.

\begin{proposition}
For any $C \geq 2$, there exists a KL-regularized linear bandit instance, such that $C^{\pistar} = C/2$ and $D^2_{\pistar} = \Theta(1)$.
\end{proposition}

\begin{proof}
We consider the function class of $\btheta \in \RR^2$ and $\|\btheta\| \leq \sqrt{2}$. The instance consists of three arms, where $\bphi(a_1) = (1,0)$, $\bphi(a_2) = (0,1)$, and $\bphi(a_3) = (1,1)$. The ground truth parameter $\btheta^* = (1, 1)$. The reference policy is given by $\piref(a_1) = \piref(a_2) = 1/2 - 1/2C$ and $\piref(a_1) = 1/C$, where $C - 1 = e^{\eta}$. A direct computation yields that 
\begin{align*}
    \pi^*(a_3) = \frac{e^\eta}{e^\eta + C - 1}, \quad \Rightarrow \quad C^{\pi^*} = C\frac{e^\eta}{e^\eta + C - 1} = \frac{C}{2}.
\end{align*}
On the other hand, we know that for $i=1,2$, we have
$D^2(a_i) \leq \piref(a_i)^{-1} \leq 4$. As for $a_3$, since we have $\la \btheta, \phi(a_3) \ra^2 = \la \btheta, \phi(a_1) + \phi(a_2) \ra^2 \leq 2 \la \btheta, \phi(a_1)\ra^2 + 2\la \btheta, \phi(a_2)\ra$, which gives that $D^2(a_3) \leq 2D^2(a_1) + 2D^2(a_2) \leq 16$. Therefore, taking expectation over $\pi^*$, we know that $D^2_{\pistar} \leq 12$ which is a constant.
\end{proof}

}

{
\color{black}
\section{Experimental Details}

\subsection{Linear Bandits}\label{app:lin-exp-details}
The linear bandit instance used for \Cref{fig:linear} has $d=20$ and $|\cA|=100$. For each arm $a \in \cA$, we randomly generate its feature vector $\bphi(a) \in \RR^d$ such that $\|\bphi(a)\|=1$. We then randomly sample the model parameter $\btheta^* \in \RR^d$ such that $\|\btheta^*\|=1$ and the expected reward is obtained via $r(a) = \la \btheta^*, \bphi(a) \ra$.

\subsection{Real-Wold Experiments}\label{app:exp-details}
MNIST consists of 60000 figures, each of which is of $28\times 28$ pixels and consists of a handwritten digit in $\{0, \cdots, 9\}$. Here, we consider each image as a context and $\cA=\{0, \cdots, 9\}$ for each context. To obtain the feature $\bphi(s,a)$, we first use the hidden representation of a classifier to embed each image as a vector in $\RR^{10}$. We then follow the approach in~\citet{zhou2020neural} to obtain the feature of each context-action pair by having $\bphi(s,a) = \xb \otimes \eb_{a+1} \in \RR^{100}$, where $\xb$ is the output of image encoder and $\otimes$ stands for tensor product.
}

\color{black}

\section{Missing Proofs from Section~\ref{sec:cb}}

\subsection{Proof of Lemma~\ref{lem:bandit:conf}}\label{app:proof-banditconf}

We first provide the following lemmas of concentration.

\begin{lemma}[{\citealt[Lemma~C.1]{zhao2024sharp}}]\label{lem:concen-behavior-bandit}
For any policy $\pi$ and state-action pairs $\{(s_i, a_i)\}_{i=1}^m$ generated i.i.d. from $\rho \times \pi$, and $\epsilon_c < 1$, with probability at least $1-\delta$, for any $g_1$ and $g_2$ we have
\begin{align*}
    \E_{\rho \times \pi}\big[\big(g_1(s,a) - g_2(s,a)\big)^2 \big] \leq \frac{2}{n}\sum_{i=1}^n \big(g_1(s_i,a_i) - g_2(s_i,a_i)\big)^2 + \frac{32}{3n}\log(2\cN_{\fcl}(\epsilon_c)/\delta) + 10 \epsilon_c,
\end{align*}
where $\cN_{\fcl}(\epsilon_c)$ is the $\epsilon_c$-covering number of $\fcl$.
\end{lemma}

\begin{lemma}[{\citealt[Lemma~C.2]{zhao2024sharp}}]\label{lem:concen-reward-bandit}
For arbitrary policy $\pi$ and dataset $\{(s_i, a_i, r_i)\}_{i=1}^m$ generated i.i.d., from the product of $\pi$, $\rho$ and the Bradley-Terry Model; let $\fls$ be the least square estimator of $\fgt$, then for any $0< \epsilon_c < 1$ and $\delta > 0$, with probability at least $1-\delta$ we have
\begin{align*}
    \sum_{i=1}^n \big(\fls(s_i,a_i) - \fgt(s_i,a_i)\big)^2 \leq 16\log (a\cN_{\fcl}(\epsilon_c)/\delta) + 4n\epsilon_c.
\end{align*}
\end{lemma}

Now we are ready to prove Lemma~\ref{lem:bandit:conf}.
\begin{proof}[Proof of Lemma~\ref{lem:bandit:conf}]
We have the following inequality
\begin{align}\label{eq:bandit-bonus-d2}
    \big(\fls(s,a) - \fgt(s,a)\big)^2 &= \frac{\big(\fls(s,a) - \fgt(s,a)\big)^2}{\E_{\piref}\big[\big(\fls(s,a) - \fgt(s,a)\big)^2 \big]}\E_{\piref}\big[\big(\fls(s,a) - \fgt(s,a)\big)^2 \big] \notag \\
    & \leq  \sup_{g_1, g_2 \in \fcl} \frac{\big(g_1(s,a) - g_2(s,a)\big)^2}{\E_{\piref}\big[\big(g_1(s,a) - g_2(s,a)\big)^2 \big]}\E_{\piref}\big[\big(\fls(s,a) - \fgt(s,a)\big)^2 \big] \notag \\
    & = D_{\fcl}^2((s,a),\piref)\E_{\piref}\big[\big(\fls(s,a) - \fgt(s,a)\big)^2 \big],
\end{align}
where the inequality holds by taking supremum to $g_1, g_2 \in \fcl$.  Now we have
\begin{align}\label{eq:bandit-bonus-concen}
    \E_{\piref}\big[\big(\fls(s,a) - \fgt(s,a)\big)^2 \big] &\leq \frac{2}{n}\sum_{i=1}^n \big(\fls(s_i,a_i) - \fgt(s_i,a_i)\big)^2 + \frac{32}{3n}\log(2\cN_{\fcl}(\epsilon_c)/\delta) + 10 \epsilon_c \notag \\
    & \leq \frac{2}{n}\big[16\log (\cN_{\fcl}(\epsilon_c)/\delta) + 4n\epsilon_c\big] + \frac{32}{3n}\log(2\cN_{\fcl}(\epsilon_c)/\delta) + 10 \epsilon_c \notag \\
    & = \frac{128}{3n}\log(2\cN_{\fcl}(\epsilon_c)/\delta) + 18 \epsilon_c,
\end{align}
where the first inequality holds due to Lemma~\ref{lem:concen-behavior-bandit} and second holds due to Lemma~\ref{lem:concen-reward-bandit}. Plugging~\eqref{eq:bandit-bonus-concen} into~\eqref{eq:bandit-bonus-d2} and setting $\epsilon_c = O(n^{-1})$ complete the proof.
\end{proof}

\subsection{Proof of Lemma~\ref{lem:taylor-expansion}}\label{app:proof-taylor-expansion}
This proof is extracted from the proof of \citet[Theorem~3.3]{zhao2024sharp} and we present it here for completeness. By definition of our objective in~\eqref{eq:kl-objective-bandit}, we have
\begin{align*} 
& J(\pi^*) - J(\pi_{g}) \\
& \quad = \EE_{(s, a) \sim \rho \times \pi^*} \bigg[\fgt(s,a) - \eta^{-1}\log \frac{\pi^*(a|s)}{\piref(a|s)}\bigg] 
- \EE_{(s, a) \sim \rho \times\pi_g} \bigg[\fgt(s, a) - \frac{1}{\eta}\log \frac{\pi_g(a|s)}{\piref(a|s)}\bigg] \\
& \quad = \frac{1}{\eta} \EE_{(s, a) \sim \rho \times\pi^*} \bigg[\log \frac{\piref(a|s) \cdot \exp\bigl(\eta \fgt( s, a)\bigr)}{\pi^*(a|s)}\bigg] 
- \frac{1}{\eta} \EE_{(s, a) \sim  \rho \times\pi_g} \bigg[\log \frac{\piref(a|s) \cdot \exp\bigl( \eta \fgt( s, a)\bigr)}{\pi_g(a|s)}\bigg]   \\
& \quad = \frac{1}{\eta} \rE_{s \sim \rho} \big[\log Z_{\fgt}(s)\big] - \frac{1}{\eta} \rE_{s \sim \rho} \big[\log Z_{g}(s)\big] - \rE_{s \sim \rho} \bigg[\sum_{a \in \cA} \pi_g(a|s) \cdot \big(\fgt( s, a) -f( s, a)\big)\bigg],
\end{align*}
where for all $g \in \fcl$ we define $Z_g(\cdot)$ as follows,
\begin{align*}
    Z_g(\cdot) \coloneqq \sum_{a\in \cA} \piref(a|\cdot)\exp\big(\eta g(\cdot,a) \big).
\end{align*}
We further denote $\Delta(s, a) = g(s, a) - \fgt( s, a)$ and $H_s(g) = \log Z_g(s) - \eta \sum_{a \in \cA} \pi_{g}(a|s) \cdot \Delta(s, a)$. It worth noticing that $\eta^{-1}\EE_{s\sim \rho}[H_s(\fgt) - H_s(g)] = J(\pi^*) - J(\pi_g)$. Now we take derivative of $H$ with respect to $\Delta(s,a)$,
\begin{align*} 
    \frac{\partial H_s(g)}{\partial \Delta(s, a)}  &= \frac{\partial}{\partial \Delta(s,a)}\bigg[\log Z_{g}(s) - \eta \sum_{a \in \cA} \pi_{g}(a|s) \cdot \Delta(s,a)\bigg] \\
    &= \frac{1}{Z_g(s)} \cdot \piref(a|s) \exp\bigl(\eta \cdot g(s,a)\bigr) \cdot \eta - \eta \cdot \pi_g(a|s) \\
    &\quad - \eta^2 \cdot \Delta(s,a) \cdot \frac{\piref(a|s) \cdot \exp\bigl(\eta \cdot g(s,a)\bigr)}{Z_g(s)} + \eta^2 \cdot \Delta(s,a) \cdot \frac{\bigl[\piref(a|s) \cdot \exp\bigl(\eta \cdot g(s,a)\bigr)\bigr]^2}{[Z_g(s)]^2} \\
    &\quad + \eta \sum_{a' \in \cA \backslash \{a\}} \frac{\piref(a'|x) \cdot \exp\bigl(\eta \cdot g(s,a')\bigr)}{Z_g(s)} \cdot \eta \cdot \Delta(s,a') \cdot \frac{\piref(a|s) \cdot \exp\bigl(\eta \cdot g(s,a)\bigr)}{Z_g(s)}
    \\&= -\eta^2 \pi_g(a|s) \Delta(s,a) + \eta^2 [\pi_g(a|s)]^2 \cdot \Delta(s,a) + \eta^2 \sum_{a' \in \cA \backslash \{a\}} \pi_g(a'|x) \pi_g(a|s) \Delta(s,a'). 
\end{align*}
Therefore, by mean value theorem, there exists $\gamma \in [0,1]$ and $g_{\gamma} = \gamma g + (1 - \gamma) \fgt$ such that
\begin{align*}
    H_s(g) - H_s(\fgt) & = -\eta^2 \gamma \sum_{a \in \cA} \pi_{g_\gamma}(a|s) \Delta(s,a)^2 + \gamma \eta^2\sum_{a_1 \in \cA} \sum_{a_2 \in \cA} \pi_{g_\gamma}(a_1|x)\pi_{g_\gamma}(a_2|x) \Delta(s,a_1)\Delta(s,a_2) \\
    & = - \eta^2 \gamma \EE_{a \sim \pi_{g_\gamma}}\big[\big(\fgt(s,a) - g(s,a)\big)^2\big] + \gamma\eta^2 \Big(\EE_{a \sim \pi_{g_\gamma}}\big[\big(\fgt(s,a) - g(s,a)\big)\big]\Big)^2 \\
    & \geq - \eta^2  \EE_{a \sim \pi_{g_\gamma}}\big[\big(\fgt(s,a) - g(s,a)\big)^2\big],
\end{align*}
where the inequality holds by omitting the second term and $\gamma \leq 1$. Now taking expectation over $\rho$, we have
\begin{align*}
    J(\pi^*) - J(\pi_g) &= \eta^{-1}\EE_{s\sim \rho}[H_s(\fgt) - H_s(g)] \\
    & \leq \eta \EE_{(s,a) \sim \rho \time \pi_{g_\gamma}}\big[\big(\fgt(s,a) - g(s,a)\big)^2\big],
\end{align*}
which concludes the proof.

\subsection{Proof of \Cref{lem:non-positive-cov}}

\begin{proof}[Proof of Lemma~\ref{lem:non-positive-cov}]
    We define $Y = -X$. Then it suffices to show that the covariance between $Y$ and $Y^2$ is
\begin{align*}
    \Cov(Y, Y^2) &= \EE[Y^3] - \EE[Y^2]\EE[Y]  \\ 
    & \geq \big(\EE[Y^2]\big)^{3/2} - \EE[Y^2]\EE[Y] \\
    & =  \big(\EE[Y^2]\big)\big( \sqrt{\EE[Y^2]} - \EE[Y] \big) \\
    & \geq 0,
\end{align*}
where both inequalities follow from Jensen's inequality.
\end{proof}

\subsection{Proof of \Cref{thm:bandit}}\label{app:proof-kl-upper}



To start with, we first define the following quantities. For all $\gamma \in [0,1]$, we define $g_\gamma \coloneqq \gamma \fps + (1 - \gamma)\fgt$ and denote
\begin{align*}
& \pi_{\gamma}(\cdot|s) \propto \piref(\cdot|s)\exp\big( \eta  g_\gamma(s, \cdot)\big), \forall s \in \cS ; \\
& G(\gamma) \coloneqq \EE_{\rho \times \pi_{\gamma}}\Big[ \big(\fps - {\fgt}\big)^2(s,a) \Big]. 
\end{align*}

The key to our analysis is the monotonicity of the function $G(\gamma)$ in $\gamma$, which is formally stated in the following lemma.

\begin{lemma}\label{lem:expansion-to-endpoint}
On event $\confbandit$, $0 \in \argmax_{\gamma \in [0, 1]} G(\gamma)$.
\end{lemma}
\begin{proof} For simplicity, we use $\triangle(s,a)$ to denote $\big(\fps - {\fgt})(s,a)$ in \emph{this} proof.
Then we know that $\triangle(s,a) \leq 0$ for all $(s,a) \in \cS \times \cA$ on event $\confbandit$. 
The most direct way to prove is to take derivative of $G$ with respect to $\gamma$,  which corresponds to the policy gradient~\citep{sutton1999policy} of $\pi_\gamma$ and thus implying a favorable structure. A direct calculation yields that
\begin{align*}
    & \ = \EE_{\rho \times \pi_\gamma}\big[\nabla_{\gamma}\log \pi_\gamma(a|s)\triangle(s,a)^2 \big] \\
    & \ = \eta \EE_{\rho}\EE_{a\sim \pi_{\gamma}}\big[ \triangle^2(s,a) \big( \triangle(s,a) - \EE_{a'\sim \pi_{\gamma}}[\triangle(s,a')] \big) \big] \\
    &\ = \eta \EE_{\rho}\Big[ \EE_{\pi_{\gamma}} \big[ \triangle^3(s,a) \big] - \EE_{\pi_{\gamma}}\big[ \triangle^2(s,a) \big]\EE_{\pi_{\gamma}}\big[ \triangle(s,a) \big] \Big] \\
    & \ \leq 0,
\end{align*}
where $\EE_{\rho}$ is the shorthand of $\EE_{s \sim \rho}$, $\EE_{\pi_{\gamma}}$ is the shorthand of $\EE_{a \sim \pi_{\gamma}}$, the first equation is derived from standard policy gradient and the inequality holds conditioned on the event $\confbandit(\delta)$ due to \Cref{lem:non-positive-cov} and \Cref{lem:non-positive-cov}.
\end{proof}


Now we are ready to prove \Cref{thm:bandit}.
\begin{proof}[Proof of \Cref{thm:bandit}]

Following the proof of \citet[Theorem~3.3]{zhao2024sharp}, we know that there exists $\bar{\gamma} \in [0, 1]$ such that
\begin{align}
    J(\pi^*) - J(\hat{\pi}) \leq \eta G(\bar{\gamma}) \leq \eta G(0),
\end{align}
where the first inequality holds due to \Cref{lem:taylor-expansion} and the second inequality holds due to the event $\confbandit$ and \Cref{lem:expansion-to-endpoint}. The term $G(0)$ can be further bounded with the $D^2$-based concentrability as follows
\begin{align}
     G(0) & = \eta \EE_{(s,a) \sim \rho\times \pi^*}\Big[ \big(\fps - {\fgt}\big)^2(s,a) \Big] \notag \\
    &\leq 4 \eta \EE_{(s,a) \sim \rho \times \pi^*} [\Gamma_n^2(s,a)] \notag \\
    & = 4\eta \beta^2 \EE_{(s,a) \sim \rho\times \pi^*}\big[D^2_{\cF}((s,a); \piref)\big] \notag \\
    &= \tilde{O}(\eta D_{\pi^*}^2 n^{-1} \log_{\fcl}(\epsilon_c)),
\end{align}
where the second inequality holds conditioned on $\confbandit(\delta)$ because of \Cref{lem:expansion-to-endpoint}, and the last inequality follows from the definition of $\confbandit(\delta)$ together with Line~\ref{line:bandit:pess}. By \Cref{lem:bandit:conf}, we know that event $\confbandit$ holds with probability at least $1- \delta$, which finishes the proof.
\end{proof}


\subsection{Proof of Theorem~\ref{thm:lowerbound-bandit}}\label{sec:lowerbound-bandit}

\begin{proof}[Proof of Theorem~\ref{thm:lowerbound-bandit}]

We consider the family of contextual bandits with $S \coloneqq |\cS|, A \coloneqq |\cA| < \infty$ and reward function in some function class $\cG$ composed of function $\cS \times \cA \to [0, 1]$ as follows.
\begin{align}
    \mathrm{CB}_{\cG} \coloneqq \{(\cS, \cA, \rho, r, \piref, \eta): r \in \cG, \rho \in \Delta(\cS), \piref \in \Delta(\cA|\cS) \}. \label{eq:contextual-bandit-class} 
\end{align}
Our goal is to prove the following statement. 
Fixing any $S \geq 1$, $\eta > 4\log 2$ and $C^* \in (2, \exp(\eta/4)]$, then for any estimator $\cD \mapsto \hat\pi \in \Delta(\cA|\cS)$, for any $n \geq 16SC^*$, there exist some function class $\cG$, such that $\exists\ \mathrm{inst} = (\cS, \cA, \rho, r, \piref, \eta) \in \mathrm{CB}_{\cG}$ with single-policy concentrability $C^{\pi^*} \leq C^*$, regularization coefficient $\eta$, $|\cS| = S = \Theta(\log|\cG|)$, and 
\begin{align} \label{eq:bandit-kl-lower-restatement}
     \subopt_{\mathrm{RKL}}(\hat\pi; \mathrm{inst}) \gtrsim \min\{\eta SC^* n^{-1}, (SC^*)^{1/2}n^{-1/2}\}.
\end{align}    
Since $\log|\cG| \geq \log \cN_{\cG}(\epsilon)$ for any $\epsilon \in (0,1)$, equation~\eqref{eq:bandit-kl-lower-restatement} yields the desired bound. 

We set $\cS = [S]$, $\cA = \{\pm1\}$, $\rho = \mathsf{Unif}(\cS)$, and the reference policy to be
\begin{align*}
  \forall s \in \cS, \piref(-1|s) = C^{-1}, \piref(+1|s) = 1-C^{-1};
\end{align*}
where $C \ge 1$ is a parameter to be specified later. We construct $2^S$ Bernoulli reward functions, in particular, $\forall \tau \in \{\pm1\}^S$, the mean function $r_\tau$ of the reward (indexed by $\tau$) is defined as
\begin{align*}
    r_\tau(s, -1) = 0.5 + \tau_s\delta, r_\tau(s, +1) = 0.5 - \alpha
\end{align*}
for any state $s \in \cS$,
where $\alpha \in (0, 1/2)$ and $\delta \in (0, 1/4]$ will be specified later.
We omit the $\mathrm{RKL}$ subscript in the following argument when it is clear in context.
By \Cref{eq:opt-exp}, the optimal policy $\pi_\tau^*$ under $r_\tau$ is
\begin{align}
    \forall s \in \cS, \pi_\tau^*(-1|s) = \frac{\exp\big(\eta(\alpha + \tau_s\delta)\big)}{\exp\big(\eta(\alpha + \tau_s\delta)\big) + C-1}, \pi_\tau^*(+1|s) = \frac{C-1}{\exp\big(\eta(\alpha + \tau_s\delta)\big) + C-1}. \label{eq:rkl:lb:opt}
\end{align}
Since $C^* \leq \exp(\eta/4)$, we assign $C = C^*$ and $\alpha = \eta^{-1}\log (C-1) \Leftrightarrow C - 1 = \exp(\eta \alpha)$, which gives
\begin{align*}
   \forall s \in \cS,& \frac{\pi_\tau^*(-1|s)}{\piref(-1|s)} \leq  C\frac{\exp(\eta (\alpha + \tau_s \delta))}{C-1 + \exp(\eta (\alpha + \tau_s \delta))} = C\frac{\exp(\eta \tau_s \delta)}{1 + \exp(\eta \tau_s \delta)} \leq C = C^*; \\
   \forall s \in \cS,&  \frac{\pi_\tau^*(+1|s)}{\piref(+1|s)} = \frac{C}{C-1} \cdot \frac{1}{\exp(\eta\tau_s\delta) + 1} \leq C = C^*;
\end{align*}
where the last inequality is due to the assumption $C^* \geq 2$. Therefore, we obtain
\begin{align}
    \max_{\tau \in \{\pm1\}^S} C^{\pi_\tau^*} \leq C^*.
\end{align}
We will abuse the notation $\subopt(\hat\pi; \tau) \coloneqq \subopt(\hat\pi; r_\tau)$.
Since $\rho = \mathsf{Unif}(\cS)$,
\begin{align}
    \subopt(\hat\pi; \tau) = \frac{1}{S}\sum_{s=1}^S \subopt_s(\hat\pi; \tau), \label{eq:rkl:lb:decomp}
\end{align}
where 
\begin{align}
    \subopt_s(\hat\pi; \tau) &= \dotp{\pi_\tau^*(\cdot|s)}{r_\tau(s,\cdot) - \eta^{-1}\log\frac{\pi_\tau^*(\cdot|s)}{\piref(\cdot|s)} } - \dotp{\hat\pi(\cdot|s)}{r_\tau(s,\cdot) - \eta^{-1}\log\frac{\hat\pi (\cdot|s)}{\piref(\cdot|s)} } \notag\\
    &= \frac{1}{\eta} \rE_{a \sim \pi^*_\tau(\cdot|s)} \bigg[\log \frac{\piref(a|s) \cdot \exp\big(\eta r_\tau(s, a)\big)}{ \pi^*_\tau(a|s)}\bigg] \notag\\
    &\quad\ -     \frac{1}{\eta} \rE_{a \sim \hat\pi(\cdot|s)} \bigg[\log \frac{\piref(a|s) \cdot \exp\big(\eta r_\tau(s, a)\big)}{ \hat\pi(a|s)}\bigg] \notag\\
    &= \frac{1}{\eta} \rE_{a \sim \pi^*_\tau(\cdot|s)} \bigg[\log \Big( \sum_{b\in\cA} \piref(b|s) \cdot \exp\big(\eta r_\tau(s, b)\big) \Big)\bigg] \notag\\
    &\quad\ -     \frac{1}{\eta} \rE_{a \sim \hat\pi(\cdot|s)} \bigg[\log \frac{\piref(a|s) \cdot \exp\big(\eta r_\tau(s, a)\big)}{ \hat\pi(a|s)}\bigg] \notag\\
    &= \frac{1}{\eta} \rE_{a \sim \hat\pi(\cdot|s)} \bigg[\log  \frac{\piref(a|s) \cdot \exp\big(\eta r_\tau(s, a)\big)}{ \pi^*_\tau(a|s)} -\log \frac{\piref(a|s) \cdot \exp\big(\eta r_\tau(s, a)\big)}{ \hat\pi(a|s)} \bigg] \notag\\
    &= \eta^{-1}\kl{\hat\pi}{\pi_\tau^*}. \label{eq:rkl:lb:kl-form}
\end{align}
We write $\tau\sim_s\tau'$ if $\tau,\tau' \in \{\pm1\}^{\cS}$ differ in only the $s$-th coordinate and $\tau\sim\tau'$ if $\exists s \in \cS, \tau\sim_s\tau'$. By \Cref{eq:rkl:lb:kl-form}, $\forall s \in \cS$, $\forall \tau, \tau' \in \{\pm1\}^{\cS}$ with $\tau \sim_s \tau'$, 
\begin{align}
&\subopt_s(\hat\pi; \tau) + \subopt_s(\hat\pi; \tau') \notag
     \\&=  \eta^{-1}\kl{\hat\pi}{\pi_\tau^*} +  \eta^{-1}\kl{\hat\pi}{\pi_{\tau'}^*} \notag\\
    &= 2\eta^{-1}\sum_{a \in \cA} \hat\pi(a|s)\log\frac{\hat\pi(a|s)}{\sqrt{\pi_\tau^*(a|s)\pi_{\tau'}^*(a|s)}} \notag\\
    &= 2\eta^{-1}\kl{\hat\pi(\cdot|s)}{\bar{\pi}_{\tau,\tau'}(\cdot|s)}   - 2\eta^{-1} \EE_{a\sim\hat\pi(\cdot|s)} \log\Big( \sum_{b \in \cA}\sqrt{\pi_\tau^*(b|s)\pi_{\tau'}^*(b|s)} \Big) \notag\\
    &\geq  - 2\eta^{-1}\log\Big( \sum_{b \in \cA}\sqrt{\pi_\tau^*(b|s)\pi_{\tau'}^*(b|s)} \Big) \notag\\
    &= \frac{1}{\eta}\log\frac{(\exp(\eta\delta) + 1)(\exp(-\eta\delta) + 1)}{4}, \label{eq:rkl:lb:pre-single}
\end{align}
where $\bar{\pi}(\cdot|s) = \sqrt{\pistar_\tau(\cdot|s)\pistar_{\tau'}(\cdot|s)} / \sum_{b \in \cA} \sqrt{\pistar_\tau(b|s)\pistar_{\tau'}(b|s)}$ for every $s \in \cS$, the inequality is due to the non-negativity of KL divergence, and the last equality follows from \Cref{eq:rkl:lb:opt} together with the design choice $C - 1 = \exp(\eta \alpha)$.

\paragraph{Case $\eta\delta \leq 2$.} Recall that  $\forall x \in \RR, (\mathrm{e}^x + \mathrm{e}^{-x})/2 - 1 = x^2\sum_{k=0}^\infty \frac{x^{2k}}{(2k+2)!} \geq x^2/2$, which implies
\begin{align}
    \Cref{eq:rkl:lb:pre-single} &= \frac{1}{\eta} \log\bigg(1 + \frac{1}{2}\Big( \frac{\mathrm{e}^{\eta\delta} + \mathrm{e}^{-\eta\delta}}{2} - 1 \Big) \bigg) \geq \frac{1}{\eta} \log\Big( 1 + \frac{\eta^2\delta^2}{4} \Big) \geq \frac{1}{\eta} \cdot \frac{\eta^2\delta^2/4}{2} = \eta\delta^2/8. \label{eq:rkl:lb:case1}
\end{align}
Here, the last inequality is due to $\eta^2\delta^2/4 \leq 1$ and $\forall x \in [0, 1], \log(1+x) \geq x/2$.

\paragraph{Case $\eta\delta > 2$.} We have $-\eta^{-1}2\log2 \geq-\delta\log2$, which implies the following bound.
\begin{align}
    \Cref{eq:rkl:lb:pre-single} &\geq \frac{1}{\eta} \log\frac{\exp(\eta\delta) + 1}{4} \geq \frac{\eta\delta - 2\log2}{\eta} = \delta - \eta^{-1}2\log2 \geq (1-\log2)\delta \geq 3\delta/10. \label{eq:rkl:lb:case2}
\end{align}
In summary, \Cref{eq:rkl:lb:case1,eq:rkl:lb:case2} imply that $\forall s \in \cS$, $\forall \tau, \tau' \in \{\pm1\}^{\cS}$ with $\tau \sim_s \tau'$,
\begin{align}
    \subopt_s(\hat\pi; \tau) + \subopt_s(\hat\pi; \tau') \geq \frac{\eta\delta^2}{8} \wedge \frac{3\delta}{10}. \label{eq:rkl:lb:single-summary}
\end{align}
Let $P_{\tau}$ be the distribution of $(s,a, y)$ where $s \sim \rho, a \sim \piref(\cdot|s)$, and $y\sim \mathsf{Bern}(r_\tau(s,a))$. Then $\forall x \in \cS$ $\forall \tau,\tau'\in\{\pm1\}^{\cS}$ with $\tau \sim_{x}\tau'$,
\begin{align}
    \kl{P_\tau}{P_{\tau'}} &=  \frac{1}{S}\sum_{s,a} \piref(a|s) \kl{\mathsf{Bern}(r_\tau(s,a))}{\mathsf{Bern}(r_{\tau'}(s,a))} \notag\\
    &= \frac{1}{S} \cdot C^{-1} \kl{\mathsf{Bern}(r_\tau(x,-1))}{\mathsf{Bern}(r_{\tau'}(x,-1))} \notag\\
    &\leq \frac{4\delta^2}{SC(0.25-\delta^2)} \leq \frac{16\delta^2}{3SC}, \label{eq:rkl:lb:single-kl}
\end{align}
where we use the requirement $\delta \leq 1/4$ and $\kl{\mathsf{Bern}(p)}{\mathsf{Bern}(q)} \le (p-q)^2/\big(q(1-q)\big)$.
Then let $P_{\cD_\tau}$ be the distribution of $\cD$ given the mean reward function $r_\tau$, we employ \Cref{eq:rkl:lb:single-kl} to get
\begin{align}
    \kl{P_{\cD_\tau}}{P_{\cD_{\tau'}}} = n\kl{P_\tau}{P_{\tau'}} \leq \frac{16n\delta^2}{3SC}. \label{eq:rkl:lb:kl}
\end{align}
Since $n \geq 16SC^* = 16SC$ by design, we can set $\delta = \sqrt{SC/n}$ (which ensures $\delta \leq 1/4$) to obtain
\begin{align}
    \sup_{\mathrm{inst}} \subopt(\hat\pi; \mathrm{inst}) &\geq \sup_{\tau \in \{\pm1\}^{\cS}} \subopt(\hat\pi; \tau) \notag\\
    &\geq \frac{1}{S} \cdot S \cdot \frac{1}{4} \cdot \Big( \frac{\eta\delta^2}{8} \wedge \frac{3\delta}{10}\Big) \min_{\tau\sim\tau'} \exp\Big( - \kl{P_{\cD_\tau}}{P_{\cD_{\tau'}}}\Big) \notag\\
    &\geq \Big( \frac{\eta SC^*}{32n} \wedge \frac{3\sqrt{SC^*}}{40\sqrt{n}}\Big) \exp(-16/3) \gtrsim \frac{\eta SC^*}{n} \wedge \sqrt{\frac{SC^*}{n}}. \notag
\end{align}
where the $S^{-1}$ in the second inequality comes from \Cref{eq:rkl:lb:decomp}, the second inequality is by substituting \Cref{eq:rkl:lb:single-summary} into Assouad's Lemma (\Cref{lem:assouad}), and the last inequality is due to \Cref{eq:rkl:lb:kl}.
\end{proof}

\section{Missing Proof from Section~\ref{sec:cb-f}}


\subsection{Proof of Theorem~\ref{thm:f-upper-bound}}

Before coming to the proof, we first introduce some useful properties. The following properties characterize the convexity of $f$-divergence when $f$ is (strongly) convex.

The strong-convexity of $f$ implies that the corresponding $f$-divergence, $D_f(\cdot || \piref)$ is also strongly convex with respect to all $\pi: \cS \rightarrow \Delta(\cA)$ supported by $\piref$. 

\begin{proposition}
Given context $s$, $D_f( \pi(\cdot|s) || \piref(\cdot|s))$ is strict convex with respect to $\pi$ if $f$ is strictly convex.
\end{proposition}

\begin{proposition}\label{prop:fdiv-sc}
Given context $s$, $\pi(\cdot|s) \mapsto D_f( \pi(\cdot|s) \| \piref(\cdot|s))$ is $4\alpha$-strong convex with respect to the metric $\mathsf{TV}$ if $f$ is $\alpha$-strongly convex.
\end{proposition}
\begin{proof}[Proof of Proposition~\ref{prop:fdiv-sc}]
We first show the gradient of $D_f$ with respect to $\pi$. 
\begin{align*}
    \frac{\partial D_f(\pi||\piref)}{\pi(a)} = \frac{\partial}{\partial \pi(a)} \sum_{b \in \cA} \piref(b)f\bigg(\frac{\pi(b)}{\piref(b)}\bigg) = f'\bigg(\frac{\pi(a)}{\piref(a)}\bigg).
\end{align*}
Now consider $\pi_1, \pi_2 \in \Delta(\cA)$ supported by $\piref$. 
\begin{align*}
& D_f(\pi_1 || \piref) - D_f(\pi_2 || \piref) - \la \pi_1 - \pi_2, \nabla D_f(\pi_2 ||\piref) \ra \\
& \quad = \sum_{a \in \cA} \piref(a)\Bigg(f\bigg(\frac{\pi_1(a)}{\piref(a)}\bigg) - f\bigg(\frac{\pi_2(a)}{\piref(a)}\bigg)\Bigg) - \sum_{a \in \cA} \big(\pi_1(a) - \pi_2(a)\big)f'\bigg(\frac{\pi_2(a)}{\piref(a)}\bigg) \\
& \quad = \sum_{a \in \cA} \piref(a)\Bigg(f\bigg(\frac{\pi_1(a)}{\piref(a)}\bigg) - f\bigg(\frac{\pi_2(a)}{\piref(a)}\bigg) - \bigg(\frac{\pi_1(a)}{\piref(a)} - \frac{\pi_2(a)}{\piref(a)}\bigg) f'\bigg(\frac{\pi_2(a)}{\piref(a)}\bigg) \Bigg) \\
& \quad \geq \frac{\alpha}{2}\sum_{a \in \cA} \piref(a) \bigg(\frac{\pi_1(a)}{\piref(a)} - \frac{\pi_2(a)}{\piref(a)}\bigg)^2 \\
& \quad = \frac{\alpha}{2}\sum_{a \in \cA} \frac{1}{\piref(a) }\big(\pi_1(a) - \pi_2(a)\big)^2 \\
& \quad  \geq \frac{\alpha}{2} \bigg(\sum_{a \in \cA} \big|\pi_1(a)-\pi_2(a)\big|\bigg)^2,
\end{align*}
where the first inequality holds due to $f$'s strong convexity and the second holds due to Cauchy–Schwarz. The proof finishes since $\|\pi_1 - \pi_2\|_1 = 2\tv{\pi_1}{\pi_2}$.
\end{proof}

We first introduce some notation and important properties concerning the convex conjugate of functions. Given some context $s$, we denote the regularization term as $H_s(\pi) = \eta^{-1}D_f(\pi(\cdot | s) \| \piref(\cdot | s))$. We use $H^*_s(r)$ to denote the convex conjugate of $H_s$, which is defined as
\begin{align*}
    H^*_s(r) = \sup_{\pi \in \cS \to \Delta^{|\cA|}} \{ \la \pi(\cdot | s), r(s, \cdot) \ra - H_s(\pi)\}.
\end{align*}
We have the following properties for the convex conjugate. The first property gives the gradient of convex conjugate (see, e.g., \citealt[Lemma 5]{zhou2018fenchel}).

\begin{proposition}
Given context $s$, and convex $f$, let $\pi_r \in \argmax_{\pi} \{ \la \pi(\cdot | s), r(s, \cdot) \ra - H_s(\pi)\}$ for some $r$, then the gradient of $H^*_s$ is given by $\nabla H^*_s( r) = \pi_r(\cdot | s)$.
\end{proposition}

We also need some properties of $\nabla^2H^*_s$, the Hessian matrix of the convex conjugate function. We first give the Hessian matrix of the original function $H_s$ as follows.
\begin{align}
    \nabla^2H_s(\pi) = \eta^{-1} \mathrm{diag}\Bigg( \frac{f''\Big(\frac{\pi(a_1|s)}{\piref(a_1|s)}\Big)}{\piref(a_1 | s)},\cdots, \frac{f''\Big(\frac{\pi(a_{|\cA|}|s)}{\piref(a_{|\cA|}|s)}\Big)}{\piref(a_{|\cA|} | s)}\Bigg). \label{eq:hessian-original}
\end{align}
Furthermore, when $f$ is $\alpha$-strongly convex, we have
\begin{align*}
    \nabla^2 H_s(\pi) \succeq \alpha \eta^{-1}  \mathrm{diag}\Big(\piref(a_1|s)^{-1}, \cdots,\piref(a_{|\cA|}|s)^{-1}\Big).
\end{align*}

The following lemma, which gives an estimate of $\nabla^2H^*_s$, is the pivot of the proof.

\begin{lemma}\label{lem:super-hessian}
For any reward $r: \cS \times \cA \to [0,1]$, we have
\begin{align*}
    \nabla^2H^*_s(r) \preceq \alpha^{-1} \eta  \mathrm{diag}\big(\piref(a_1|s), \cdots,\piref(a_{|\cA|}|s)\big).
\end{align*}
\end{lemma}

\begin{proof}[Proof of Lemma~\ref{lem:super-hessian}]
Given reward function $r: \cS \times \cA \to [0,1]$, we consider
\begin{align*}
    \pi_r \in \argmax_{\pi \in \cS \to \Delta^{|\cA|}} \{ \la \pi(\cdot |s), r(\cdot |s) \ra - H_s(\pi)\}.
\end{align*}
From \Cref{eq:hessian-original} we know that $\nabla^2H_s(\pi_r)$ is invertible. Therefore, by \citealt[Proposition~3.2]{penot1994sub}, we have $\nabla^2H^*_s(r) \preceq (\nabla^2H_s(\pi_r))^{-1}$. Since $f$ is $\alpha$-strongly convex, we have 
\begin{align*}
    \nabla^2H^*_s(r) \preceq \alpha^{-1} \eta  \mathrm{diag}\big(\piref(a_1|s), \cdots,\piref(a_{|\cA|}|s)\big),
\end{align*}
which finishes the proof.
\end{proof}

Now we are ready to prove Theorem~\ref{thm:f-upper-bound}.

\begin{proof}[Proof of Theorem~\ref{thm:f-upper-bound}]
Consider our estimation $\fls$ which approximates the ground truth reward function $\fgt$, we know that 
\begin{align*}
    \hat \pi = \argmax_{\pi \in \cS \to \Delta(\cA)} \Big\{\EE_{(s,a) \sim \rho \times \pi}[\fls(s,a)] - \eta^{-1}\EE_{s \sim \rho}\big[D_{f}(\pi || \piref)\big] \Big\}.
\end{align*}
We have the following sub-optimality decomposition
\begin{align*}
    J(\pi^*) - J(\hat \pi) &= \EE_{s \sim \rho} \Big[\EE_{a \sim \pi^*}[\fgt(s,a)] - \EE_{a \sim \hat \pi}[\fgt(s,a)] - \eta^{-1}\big[D_{f}(\pi^* \| \piref) - D_{f}(\hat \pi \| \piref) \big]\Big]  \\
    & = \EE_{s \sim \rho}\big[H^*_s(\fgt) - H^*_s(\fls) - \la \hat \pi, \fgt - \fls \ra \big] \\
    & =  \EE_{s \sim \rho}\big[H^*_s(\fgt) - H^*_s(\fls) - \la \nabla H^*_s(\fls), \fgt - \fls \ra \big] \\
    & = \EE_{s \sim \rho}[(\fgt - \fls)^\top \nabla^2H^*_s(\tilde g)(\fgt - \fls)],
\end{align*}
where $\tilde g = \gamma \fgt + (1-\gamma)\fls$ and $\gamma \in [0,1]$ and the last equation holds due to Taylor's expansion. Now, for any $\delta \in (0,1)$ and $\epsilon_c > 0$,  with probability at least $1 - \delta$
\begin{align*}
    J(\pi^*) - J(\hat \pi) & = \EE_{s \sim \rho}[(\fgt - \fls)^\top \nabla^2H^*_s(\tilde g)(\fgt - \fls)] \\
    & \leq \alpha^{-1} \eta \EE_{s \sim \rho}\Big[(\fgt - \fls)^\top \mathrm{diag}\big(\piref(a_1|s), \cdots,\piref(a_{|\cA|}|s)\big)(\fgt - \fls)\Big] \\
    & = \alpha^{-1} \eta \EE_{(s,a) \sim \rho \times \piref}\big[\big(\fgt(s,a) - \fls\big(s,a))^2\big] \\
    & \leq \alpha^{-1} \eta \bigg( \frac{128}{3n}\log(2\cN_{\fcl}(\epsilon_c)/\delta) + 18 \epsilon_c \bigg),
\end{align*}
where the first inequality holds due to Lemma~\ref{lem:super-hessian} and last inequality holds due to equation~\eqref{eq:bandit-bonus-concen}. Setting $\epsilon_c = O(n^{-1})$ completes the proof.
\end{proof}

\subsection{Proof of Theorem~\ref{thm:f-lower-bound}}\label{sec:f-lower-bound}


We first provide the following lemma that gives the close form of optimal policy under $\chi^2$-divergence regularization.

\begin{lemma}[{\citet[Lemma~G.2]{huang2025best}}]\label{lem:fdiv:opt}
Let $\pi^*$ be the optimal policy of $\chi^2$-divergence regularized objective with reward function $r$, then $\pi^*$ has the closed form 
\begin{align*}
    \pi^*(\cdot) = \piref(\cdot)\max\big\{0, \eta(r(\cdot) - \lambda)\big\}, \text{ where }\sum_{a\in \cA}\pistarfdiv(a) = 1.
\end{align*}
\end{lemma}
By \Cref{prop:fdiv-sc}, $\pistarfdiv = \argmax_{\pi \in \Delta(\cA)} \objfdiv(\pi)$ is unique.
The sub-optimality gap for $f$-divergence is consequently defined as
\begin{align}
    \suboptfdiv(\cdot) \coloneqq \suboptfdiv(\cdot; \cA, r, \piref) = \objfdiv(\pistarfdiv) - \objfdiv(\cdot).
\end{align}

Now we are ready to prove Theorem~\ref{thm:f-lower-bound}.

\begin{proof}[Proof of Theorem~\ref{thm:f-lower-bound}]
We still consider the family of contextual bandits $\mathrm{CB}_{\cG}$ given by~\eqref{eq:contextual-bandit-class}. We, still, aim to prove the following statement. 
Fixing any $S \geq 32\log 2$, $\eta > 4\log 2$ and $\alpha$, we set $f(x) \coloneqq \alpha (x-1)^2/2$, then for any estimator $\cD \mapsto \hat\pi \in \Delta(\cA|\cS)$, for any $n$ sufficiently large, there exist some function class $\cG$, such that $\exists\ \mathrm{inst} = (\cS, \cA, \rho, r, \piref, \eta) \in \mathrm{CB}_{\cG}$ with $|\cS| = S = \Theta(\log|\cG|)$, and 
\begin{align} \label{eq:bandit-f-lower-restatement}
     \suboptfdiv(\hat\pi; \mathrm{inst}) \gtrsim \alpha^{-1}\eta S n^{-1}.
\end{align}    
Since $\log|\cG| \geq \log \cN_{\cG}(\epsilon)$ for any $\epsilon \in (0,1)$, equation~\eqref{eq:bandit-f-lower-restatement} yields the desired bound. 

We again omit subscripts $f$div when it is clear in context. We set $\cS = [S]$, $\cA = \{-1, +1\}$, and $\rho = \mathsf{Unif}(\cS)$. For all $s \in \cS$, $\piref = \mathsf{Unif}(\cA)$. We further consider the following reward function class. We leverage \Cref{lem:max-signals} and obtain a set $\cV \in \{-1,+1\}^S$ such that (1) $|\cV| \geq \exp(S/8)$ and (2) for any $v, v' \in \cV, v \neq v'$ , one has $\|v - v'\|_1 \geq S/2$. We construct the following reward function class where the reward follows Bernoulli distribution and the mean functions are given by the function class
\begin{align*}
    \cG = \{ r_{v}(s, -1) = 1/2 + v_s\delta, r_{v'}(s,+1) = 1/2 + v_s'\delta, \ \forall s \in \cS  | v \in \cV\},
\end{align*}
where $\delta \in (0, \eta^{-1}\alpha]$ is to be specified later. Fix some context $s$ and $v_1 \neq v_2$ different at entry $s$ and corresponding reward $r_1$ and $r_2$. Without loss of generality, we assume $r_1(s, \cdot) = (1/2 + \delta, 1/2 - \delta)$ and $r_2(s, \cdot) = (1/2 - \delta, 1/2 + \delta )$. Then direct calculation implies that 
\begin{align*}
    & \pi_1^*(\cdot|s) = \frac{1}{2}\max\{0, \eta\alpha^{-1}(r_1(s, \cdot) - {\lambda})\} = 0.5  \eta\alpha^{-1}(r_1(s, \cdot) - {\lambda}), \\
    & \pi_2^*(\cdot|s) = \frac{1}{2}\max\{0, \eta\alpha^{-1}(r_2(s, \cdot) - {\lambda})\} = 0.5  \eta\alpha^{-1}(r_2(s, \cdot) - {\lambda}),
\end{align*}
where $\lambda = 0.5 - \eta^{-1}\alpha$. 
Note that $2{\chi^2}(\mu\|\nu) + 1= \sum_{a\in\cA}[\mu(a)]^2/\nu(a)$ and ${\chi^2} = D_f$, we obtain that $\forall \ \hat{\pi}$,
\begin{align}
    & \subopt_s(\hat\pi(\cdot | s); r_1) + \subopt_s(\hat\pi(\cdot | s); r_2) \\
    & \quad = \dotp{r_1(s, \cdot)}{\pi^*_1(\cdot|s)} + \dotp{r_2(s, \cdot)}{\pi^*_2(\cdot|s)} - \overbrace{\dotp{r_1(s, \cdot) + r_2(s, \cdot)}{\hat\pi(\cdot | s)}}^{=1} + \overbrace{2\eta^{-1}\alpha {\chi^2}(\hat\pi(\cdot | s)\|\piref(\cdot | s))}^{\geq0} \notag\\
    &\quad \quad - \eta^{-1}\alpha \cdot {\chi^2}(\pistar_1(\cdot|s)\|\piref(\cdot | s)) - \eta^{-1}\alpha \cdot {\chi^2}(\pistar_2(\cdot|s)\|\piref(\cdot | s)) \notag\\
    &\quad \geq 2\dotp{r_1(s, \cdot)}{\pistar_1(\cdot|s)} - 1 - 2\eta^{-1}\alpha \cdot {\chi^2}(\pistar_1(\cdot|s)\|\piref(\cdot|s)) \notag\\
    &\quad = 1 + \frac{2\eta \delta^2}{\alpha } - 1 - \frac{\eta \delta^2}{\alpha } = \frac{\eta \delta^2}{\alpha }. \label{eq:fdiv:pre-lb}
\end{align}
Now we take expectation over all possible contexts and recall that $\|v - v'\|_1 \geq S/2$ for $v \neq v'$, we know that for any $r_1 \neq r_2 \in \cG$
\begin{align*}
    \subopt(\hat\pi; r_1) + \subopt(\hat\pi; r_2) \geq \frac{\eta \delta^2}{2\alpha }
\end{align*}
Given any mean reward function $r \in \cG$, let $P_{r}$ be the distribution of $(s, a, \mathtt{r})$ when $s \sim \rho$, $a \sim \piref(\cdot|s)$, and $\mathtt{r} \sim \mathsf{Bern}(r(s, a))$. Suppose $P_{\cD_r}$ is the distribution of the dataset
given mean reward function $r$, then $\kl{P_{\cD_{r_1}} }{P_{\cD_{r_2}}} = n \kl{P_{{r_1}} }{P_{{r_2}}}$ for any pair of $r_1, r_2 \in \cG$.
Now we invoke Fano's inequality (Lemma~\ref{lem:fano}) to obtain 
\begin{align*}
     \inf_{\pi}\sup_{\mathrm{inst} \in \mathrm{CB}_{\cG}}  \subopt(\hat\pi; \mathrm{inst}) & \geq \frac{\eta \delta^2}{4 \alpha}\Bigg(1 - \frac{\max_{r_1\neq r_2 \in \cG}\kl{P_{\cD_{r_1}} }{P_{\cD_{r_2}}} + \log 2}{\log |\cG|} \Bigg) \\
     & \geq \frac{\eta \delta^2}{4 \alpha}\Bigg(1 - \frac{64 n\delta^2 + 8\log 2}{S} \Bigg),
\end{align*}
where the second inequality holds due to $\KL(\mathsf{Bern}(p)\|\mathsf{Bern}(q)) \leq (p-q)^2/[q(1-q)]$. Let $\delta = 16^{-1}\sqrt{S n^{-1}}$, then we obtain that for all $\pi$ we have  
\begin{align*}
    \sup_{\mathrm{inst} \in \mathrm{CB}_{\cG}} \subopt(\hat\pi; \mathrm{inst}) \gtrsim \frac{\eta S}{\alpha n},
\end{align*}
which finishes the proof in that $\log_2 |\cG| = S$.
\end{proof}

\section{Generalization to contextual dueling bandits}\label{app:cdb}

In this section, we extend our algorithm to the problems of regularized contextual dueling bandits,
where the learner receives preference comparison instead of absolute signals. Our setup largely follows \citet{zhu2023principled,zhan2023provable} and the notion of sub-optimality follows \citet{xiong2024iterative,zhao2024sharp}.

\subsection{Problem setup}\label{sec:duel-setup}

We still consider contextual bandits $(\cS, \cA, r, \piref)$ where $\cS$ is the state space, $\cA$ is the action space and $r: \cS \times \cA \to [0, 1]$ is the reward function.\footnote{We overload some notations in \Cref{sec:cb} by their dueling counterparts for notational simplicity.} But only relative preference feedback is available, viz., we have an $\iid$ offline dataset $\cD = \{(s_i,a^1_i, a^2_i, y_i)\}_{i=1}^n$, where $s_i \in \cS$ is generated from distribution $\rho$ and $a^1_i, a^2_i \sim \piref$. The binary preference label $y_i=1$ indicates $a_i^1$ is preferred over $a_i^2$ (denoted by $a^1 \succ a^2$) and $0$ for $a^2 \succ a^1$ given
context $s$. 
In this work we consider the Bradley-Terry Model, where $\mathbb{P}[y=1|s,a^1, a^2] = \sigma(r(s_i, a^1_i) - r(s_i, a^2_i))$, where $\sigma(x) = (1+e^{-x})^{-1}$ is the link function.
The objective here identical to~\eqref{eq:kl-objective-bandit} for KL-regularization and~\eqref{eq:f-objective-bandit} for $f$-divergence regularization. Our goal is still to find an $\epsilon$-optimal policy. To control the complexity of the function class $\fcl$, we assume that \Cref{assume:general-function-approx} still holds here.

\paragraph{Concentrability.}
Analogous to \Cref{sec:cb}, we need our estimation from offline dataset generalizable to the state-action pairs visited by our obtained policy. While density-ratio-based concentrability can be directly adapted to dueling bandit, we need a slightly different notion of $D^2$-divergence. This is because in dueling bandit, we cannot observe the absolute reward and best estimation $g$ we can achieve is that for any state $s$ and actions $a^1, a^2$, our estimated $g(s,a^1) - g(s,a^2) \approx r(s,a^1) - r(s,a^2)$. This implies that there exists some mapping $b: \cS \to [-1, 1]$ such that $g(s,a) - b(s) \approx r(s,a)$ on the offline data, which leads to the following definition.


\begin{definition}\label{def:dueling-bandit:D-sq}
Given a class of functions $\fcl \subset (\cS \times \cA \to \mathbb{R})$ and some policy $\pi$, let $\cB = (\cS \to [-1, 1])$ be the function class, define the $D^2$-divergence $D^2_{\fcl}((s,a);\pi)$ as
\begin{align*}
    \sup_{ g,h \in \fcl} \inf_{b \in \cB} \frac{\big( g(s, a) - h(s, a) - b(s)\big)^2}{\rE_{s\sim \rho}\Var_{a' \sim \pi(\cdot | s')}[g(s', a') - h(s', a')]}.
\end{align*}
\end{definition}

A similar definition has been introduced in \citet[Definition~2.6]{zhao2024sharp}, which underpins the following two assumptions that characterize the coverage ability of $\piref$ similarly as in \Cref{sec:cb}.


Given a reference policy $\piref$, we define two coverage notions for contextual dueling bandits.
\begin{assumption}[All-policy concentrability]\label{assume:all-coverage-dueling}
    $D^2 \coloneqq \sup_{(s,a) \in \cS \times \cA} D^2_{\fcl}((s,a); \piref) < \infty$.
\end{assumption}
\begin{assumption}[Single-policy concentrability]\label{assume:single-coverage-dueling}
    $ D_{\pi^*}^2 \coloneqq \E_{(s,a) \sim \rho \times \pi^*}[D^2_{\fcl}((s,a); \piref)] < \infty$.
\end{assumption}
Similar single-policy concentrability assumptions have appeared in previous work in offline contextual dueling bandits~\citep{huang2024correcting,song2024importance} and similar notions has also appeared in the analysis of model-based RL \citep{uehara2021pessimistic,wang2024model}. Still, while \Cref{assume:single-coverage-dueling} is strictly weaker than \Cref{assume:all-coverage-dueling}, in general cases, the two quantities, $C^{\pi^*}$ and $D^2_{\pi^*}$ cannot be bounded by each other.

\subsection{Algorithms and Results}


\subsubsection{Algorithms for KL-regularized contextual dueling bandits}

\begin{algorithm*}[t]
	\caption{Offline KL-Regularized Pessimistic Contextual Dueling Bandit (\algcdb)}\label{algorithm:dueling-bandit}
	\begin{algorithmic}[1]
	\REQUIRE regularization $\eta$, reference policy $\piref$, function class $\fcl$, offline dataset $\cD = \{(s_i,a^1_i, a^2_i, y_i)\}_{i=1}^n$
\STATE Compute the maximum likelihood estimator of the reward function
    \begin{align*}
        \fls = \argmin_{g \in \fcl} \sum_{i=1}^n \Big[ y_i \log \sigma \Big(\big[g(s_i, a^1_i) - g(s_i, a^2_i)\big]\Big) + (1-y_i) \log \sigma \Big(\big[g(s_i, a^2_i) - g(s_i, a^1_i)\big]\Big) \Big]
    \end{align*}\label{line:dueling-mle}
    \STATE Let $\fps(s, a) = \fls(s, a) - \Gamma_n(s,a)$, where $\Gamma_n(s, a)$ is the bonus term in \Cref{eq:bonus-dueling-bandit} \label{line:dueling-bandit:pess}
    \ENSURE $\hat \pi(a|s) \propto \piref(a|s) \exp \big(\eta \cdot \fps(s, a)\big)$
\end{algorithmic}
\end{algorithm*}

We elucidate $\algcdb$ for offline KL-regularized contextual dueling bandits, whose pseudocode is summarized in \Cref{algorithm:dueling-bandit}. $\algcdb$ first estimate the ground truth function $\fgt$ on offline dataset with maximum likelihood estimator (MLE) to estimate a function $\fls \in \fcl$. After that, analogous to \Cref{algorithm:bandit-pess}, we adopt the principle of pessimism in the face of uncertainty. Specifically, we define the penalty term
\begin{align}
    \Gamma_n(s,a)= \beta \sqrt{ D_{\fcl}^2((s,a),\piref)}, \label{eq:bonus-dueling-bandit}
\end{align}
where
\begin{align}
    \beta^2 = 128\log(2\cN_{\fcl}(\epsilon_c)/\delta)/3n + 18 \epsilon_c =\tilde{O}(n^{-1}) \label{eq:conf-radis-dueling-bandit}
\end{align}
and then subtract it from the MLE $\fls$ to obtain a pessimistic estimator $\fps$.
$\algcb$ then output the policy $\hat{\pi}$, maximizing the estimated objective
\begin{align*}
    \hat{J}(\pi) =\EE_{(s,a) \sim \rho \times \pi} \bigg[\fps(s,a) - \eta^{-1} \log \frac{\pi(a|s)}{\piref(a|s)}\bigg],
\end{align*}
the maximizer of which is in closed form as the counterpart of \Cref{eq:opt-exp}.
\begin{align*}
    \hat \pi(a|s) \propto \piref(a|s) \exp \big(\eta \cdot \fps(s,a)\big).
\end{align*}

We provide the following theoretical guarantees for \Cref{algorithm:dueling-bandit}.

\begin{theorem}\label{thm:main-dueling-bandit}
Under \Cref{assume:single-coverage-dueling}, if we set $\Gamma_n$ according to \Cref{eq:bonus-dueling-bandit}, then for sufficiently small $\epsilon \in (0, 1)$, with probability at least $1 - \delta$, $n = \tilde{O}\big(\eta (D_{\pi^*}^2 \wedge C^{\pi^*})\epsilon^{-1}\big)$ is sufficient to guarantee the output policy $\hat \pi$ of \Cref{algorithm:dueling-bandit} to be $\epsilon$-optimal.
\end{theorem}

\begin{remark}
\citet{zhao2024sharp} achieved an $\tilde{O}(\epsilon^{-1})$ sample complexity under \Cref{assume:all-coverage-dueling}. Comparing to \citet{zhao2024sharp}, $\algcdb$ achieves the same $\tilde{O}(\epsilon^{-1})$ sample complexity but only requiring \Cref{assume:single-coverage-dueling}, which is weaker than \Cref{assume:all-coverage-dueling}.
\end{remark}

The following theorem provides the sample complexity lower bound for KL-regularized dueling contextual bandits.

\begin{theorem}\label{thm:lowerbound-dueling}
For any sufficiently small $\epsilon \in (0, 1), \eta > 0$, $1 \leq C^* \leq \exp(\eta/2)/2$, and any algorithm \textsf{Alg}, there is a KL-regularized contextual dueling bandit instance with single-policy concentrability $C^{\pi^*} \leq C^*$ such that \textsf{Alg} requires at least $\Omega\big(\min\{{\eta C^* \log \cN_{\fcl}(\epsilon_c)}/{\epsilon}, {\log \cN_{\fcl}(\epsilon_c)}(C^*)^2/{\epsilon^2}\}\big)$ samples to return an $\epsilon$-optimal policy.
\end{theorem}



\begin{remark}
Theorem~\ref{thm:lowerbound-dueling} shows that when $\epsilon$ is sufficiently small, any algorithm for offline KL-regularized contextual dueling bandits requires at least $\Omega(\eta C^{\pi^*} \log \cN_{\fcl}(\epsilon)\epsilon^{-1})$ samples to output an $\epsilon$-optimal policy, which matches the sample complexity upper bound in Theorem~\ref{thm:main-dueling-bandit}, indicating that $\algcb$ is nearly optimal.
\end{remark}

\subsubsection{Algorithm and Results for $f$-divergence Regularized CDBs}

\begin{algorithm*}[t]
	\caption{Offline $f$-Divergence Regularized Contextual Dueling Bandits (\algfcdb)}\label{algorithm:dueling-bandit-fdiv}
	\begin{algorithmic}[1]
	\REQUIRE regularization $\eta$, reference policy $\piref$, function class $\fcl$, offline dataset $\cD = \{(s_i,a^1_i, a^2_i, y_i)\}_{i=1}^n$
\STATE Compute the maximum likelihood estimator of the reward function
    \begin{align*}
        \fls = \argmin_{g \in \fcl} \sum_{i=1}^n \Big[ y_i \log \sigma \Big(\big[g(s_i, a^1_i) - g(s_i, a^2_i)\big]\Big) + (1-y_i) \log \sigma \Big(\big[g(s_i, a^2_i) - g(s_i, a^1_i)\big]\Big) \Big].
    \end{align*}\label{line:dueling-mle:fdiv}
    \STATE Compute the optimal policy with respect to reward $\fls$ 
    \begin{align*}
        \hat \pi(\cdot | s) \leftarrow \argmax_{\pi(\cdot | s) \in \Delta(\cA)} \sum_{a \in \cA} \pi(a|s)\fls(s, a) + \eta^{-1} D_f\big(\pi (\cdot|s) \| \piref(\cdot|s)\big)
    \end{align*}
    \ENSURE $\hat \pi(a|s)$
\end{algorithmic}
\end{algorithm*}

We present an offline learning algorithm for $f$-divergence regularized contextual dueling bandit, $\algfcdb$, in \Cref{algorithm:dueling-bandit-fdiv}. $\algfcdb$ first leverages maximum likelihood estimator to find a function $\fls \in \fcl$ that minimizes its risk on the offline dataset.  Then the algorithm constructs the output policy $\hat{\pi}$ that maximizes the f-divergence regularized objective induced by $\fls$. Similar to Algorithm~\ref{algorithm:bandit-fdiv}, we do not require any pessimism in $\algfcdb$. The following theorem provides an upper bound of \Cref{algorithm:dueling-bandit-fdiv}. 

\begin{theorem}\label{thm:dueling-bandit-f-div-upper}
For any sufficiently small $\epsilon \in (0, 1)$, and $\eta, \alpha >0$, with probability at least $1 - \delta$, $n = \tilde{O}(\alpha^{-1}\eta \log {\cN}(\epsilon)\epsilon^{-1})$ is sufficient to guarantee that the output policy $\hat \pi$ of \Cref{algorithm:dueling-bandit-fdiv} is $\epsilon$-optimal.
\end{theorem}

The following theorem provides a lower bound for offline $f$-divergence regularized contextual dueling bandit with strongly convex $f$.

\begin{theorem}\label{thm:dueling-bandit-f-div-lower}
For any $\epsilon \in (0, 1)$, $\alpha, \eta > 0$, and offline RL algorithm \textsf{Alg}, there is an $\alpha$-strongly convex $f$ and $f$-divergence regularized contextual dueling bandit instance such that \textsf{Alg} requires at least $\Omega\big(\alpha^{-1}\eta \log{\cN}(\epsilon) \epsilon^{-1}\big)$ samples to return an $\epsilon$-optimal policy.
\end{theorem}

\begin{remark}
Theorem~\ref{thm:dueling-bandit-f-div-lower} indicates that, when $\epsilon$ is sufficiently small, to produce an $\epsilon$-optimal policy, any algorithm for offline $f$-regularized contextual bandits with strongly convex $f$ requires at least $\tilde{\Omega}(\alpha^{-1}\eta\epsilon^{-1})$ samples. This lower bound matches the sample-complexity upper bound in Theorem~\ref{thm:dueling-bandit-f-div-upper}, indicating that \Cref{algorithm:dueling-bandit-fdiv} is nearly optimal.
\end{remark}

\section{Missing proof from \Cref{app:cdb}}

\subsection{Proof of Theorem~\ref{thm:main-dueling-bandit}}

The proof follows the proof in Section~\ref{sec:cb}. At the beginning, we first define the event $\confduelbandit(\delta)$ given $\delta>0$ as
\begin{align}\label{eq:conf-dueling-bandit}
    \confduelbandit(\delta) \coloneqq \Big\{ \exists \ b: \cS \to [-1,1], \forall (s,a) \in \cS \times \cA, 
   \big| \fls(s,a) - b(s) - \fgt(s,a)\big| \leq \Gamma_n(s,a) \Big\}.
\end{align}
Here, $\Gamma_n$ is defined in~\eqref{eq:bonus-dueling-bandit}. We abuse the notation and define $b(\cdot)$ as
\begin{align}
    b = \argmin_{\cB} \sup_{(s,a) \in \cS \times \cA} \Phi_b(s,a) - \Gamma_n(s,a), \label{eq:def-bias-dueling}
\end{align}
where $\Phi_b(s,a) = \big| \fls(s,a) - b(s) - \fgt(s,a)\big|$ and when $\confduelbandit$ holds, for all $(s,a) \in \cS \times \cA$, we have $\Phi_b(s,a) \leq \Gamma_n(s,a)$
This indicates that the least square estimation $\fls$ obtained in Line~\ref{line:dueling-mle} of Algorithm~\ref{algorithm:dueling-bandit}, after adjusted by some bias function $b$, is close to the true function $\fgt$. The following lemma shows that this event holds with high probability. 

\begin{lemma}\label{lem:dueling-bandit:conf}
For any $\delta > 0$, $\PP(\confbandit(\delta)) \geq 1-\delta$.
\end{lemma}

\begin{proof} From Lemma \ref{lem:concentration-dueling}, we have that with probability at least $1 - \delta$, it holds that \begin{align}
    \EE_{s'\sim \rho}\Var_{a' \sim \piref(\cdot |s')} \big[\fls(s', a') - \fgt(s',a')\big] \leq O\bigg(\frac{1}{n}\log(\cN_{\fcl}(\epsilon_c)/\delta) + \epsilon_c\bigg). \label{eqn:dueling-conf:1}
\end{align}

It further holds true that for some $b: \cS \to \RR$\begin{align}
D_\fcl^2((s, a), \piref)) \cdot \EE_{s\sim \rho}\Var_{a \sim \piref(\cdot |s)} \big[\fls(s, a) - \fgt(s,a)\big] \geq \big(\fls(s,a) - b(s) - \fgt(s,a)\big)^2. \label{eqn:dueling-conf:2}
\end{align}
Substituting \eqref{eqn:dueling-conf:1} into \eqref{eqn:dueling-conf:2}, we have
\begin{align}\label{eq:dueling-bandit-bonus-d2}
    & \inf_{b} \big(\fls(s,a) - b(s) - \fgt(s,a)\big)^2 \\
    & \quad = \inf_{b} \frac{\big(\fls(s,a) - b(s) - \fgt(s,a)\big)^2}{\EE_{s'\sim \rho}\Var_{a' \sim \piref(\cdot |s')} \big[\fls(s', a') - \fgt(s',a')\big]}\EE_{s'\sim \rho}\Var_{a' \sim \piref(\cdot |s')} \big[\fls(s', a') - \fgt(s',a')\big] \notag \\
    & \quad \leq D_{\fcl}^2((s,a),\piref)\E_{\piref}\big[\big(\fls(s,a) - b(s) - \fgt(s,a)\big)^2 \big] \\
    & \quad \leq D_{\fcl}^2((s,a),\piref) O\bigg(\frac{1}{n}\log(\cN_{\fcl}(\epsilon_c)/\delta) + \epsilon_c\bigg),
\end{align}
where the first inequality holds due to the definition of $D_{\fcl}^2((s,a),\piref)$ and the last inequality holds due to Lemma~\ref{lem:concentration-dueling}.
\end{proof}

We overload the following quantities. For any $\gamma \in [0,1]$ and $(s,a) \in \cS \times \cA$, we define
\begin{align*}
    g_\gamma(s,a) \coloneqq \gamma (\fps(s,a) - b(s)) + (1 - \gamma)\fgt(s,a).
\end{align*}
Furthermore, we introduce the following quantities
\begin{align}
& \pi_{\gamma}(\cdot|\cdot) = \pi_{g_\gamma}(\cdot|\cdot) \propto \piref(\cdot|\cdot)\exp\big( \eta  g_\gamma(\cdot, \cdot)\big), \notag \\
& G(\gamma) \coloneqq \EE_{\rho \times \pi_{\gamma}} \big[ \big(\fps(s,a) - b(s) - {\fgt}(s,a)\big)^2 \big], \notag
\end{align}
where $b(\cdot)$ is defined in~\eqref{eq:def-bias-dueling}. We still have the monotonicity of the function $G(\gamma)$, which is characterized by the following lemma.

\begin{lemma}\label{lem:expansion-to-endpoint-dueling}
On event $\confduelbandit(\delta)$, $0 \in \argmax_{\gamma \in [0, 1]} G(\gamma)$.
\end{lemma}

\begin{proof} For simplicity, we use $\triangle(s,a)$ to denote $\fps(s,a) - b(s) - {\fgt}(s,a)$ in \emph{this} proof.
Then on event $\confduelbandit(\delta)$, we know that $\triangle(s,a) \leq 0$ for all $(s,a) \in \cS \times \cA$. 
Taking derivatives of $G$ w.r.t., $\gamma$ directly, we conclude that for all $\gamma \in [0, 1]$,
\begin{align*}
    G'(\gamma) & = \eta \EE_{\rho}\EE_{a\sim \pi_{\gamma}}\big[ \triangle^2(s,a) \big( \triangle(s,a) - \EE_{a'\sim \pi_{\gamma}}[\triangle(s,a')] \big) \big] \\
    & = \eta \EE_{\rho}\Big[ \EE_{\pi_{\gamma}} \big[ \triangle^3(s,a) \big] - \EE_{\pi_{\gamma}}\big[ \triangle^2(s,a) \big]\EE_{\pi_{\gamma}}\big[ \triangle(s,a) \big] \Big] \\
    &  \leq 0,
\end{align*}
where $\EE_{\rho}$ is the shorthand of $\EE_{s \sim \rho}$, $\EE_{\pi_{\gamma}}$ is the shorthand of $\EE_{a \sim \pi_{\gamma}}$ and the inequality holds conditioned on the event $\confduelbandit(\delta)$ due to Lemma~\ref{lem:non-positive-cov}.
\end{proof}

Finally, we have the proposition that adding some bias term $b: \cS \to \RR$ does not affect the resulting policy.

\begin{proposition} \label{prop:dueling-add-bias}
Let $b: \cS \to \RR$ be some bias function, then for all $g \in \fcl$ we have $J(\pi_g) = J(\pi_{g - b})$, where $(g-b)(s,a)=g(s,a) - b(s)$.
\end{proposition}

\begin{proof} For any fixed state $s \in \cS$, we have for any $a \in \cA$ that,
\begin{align*}
    \pi_g(a|s) & = \frac{\piref(a|s)\exp\big(\eta g(s,a)\big)}{ \sum_{a'\in \cA} \piref(a'|s)\exp\big(\eta g(s,a') \big)} \\
    & = \frac{\piref(a|s)\exp\big(\eta g(s,a)\big) \exp\big(\minus \eta b(s)\big)}{ \sum_{a'\in \cA} \piref(a'|s)\exp\big(\eta g(s,a') \big) \exp\big(\minus \eta b(s)\big)} \\
    & = \frac{\piref(a|s)\exp\big(\eta [g(s,a) - b(s)] \big)}{ \sum_{a'\in \cA} \piref(a'|s)\exp\big(\eta [g(s,a') - b(s)] \big) } \\
    & = \pi_{g-b}(a|s),
\end{align*}
which indicates that $\pi_g = \pi_{g-b}$. This immediately leads to  $J(\pi_g) = J(\pi_{g - b})$.
\end{proof}


Now we are ready to prove Theorem~\ref{thm:main-dueling-bandit}.

\begin{proof}[Proof of Theorem~\ref{thm:main-dueling-bandit}]
We proceed the proof under the event $\confduelbandit(\delta)$. By Proposition~\ref{prop:dueling-add-bias}, we know that
\begin{align*}
    J(\pi^*) - J(\hat{\pi}) & = J(\pi^*) - J(\pi_{\fps}) \\
    & = J(\pi^*) - J(\pi_{\fps-b}).
\end{align*}
Consequently, there exist some $\gamma \in [0,1]$ and $b: \cS \to [-1,1]$ such that
\begin{align}
    J(\pi^*) - J(\hat{\pi}) & = J(\pi^*) - J(\pi_{\fps-b}) \notag \\
    & \leq \eta \EE_{\rho \times \pi_{\gamma}} \big[ \big(\fps(s,a) - b(s) - {\fgt}(s,a)\big)^2 \big] \notag \\
    & = \eta G(\gamma), \label{eq:dueling-taylor-expansion}
\end{align}
where the inequality holds due to Lemma~\ref{lem:taylor-expansion}. 
Under event $\confduelbandit(\delta)$, we know that $\fps(s,a) - b(s) \leq {\fgt}(s,a)$. Together with Lemma~\ref{lem:expansion-to-endpoint-dueling}, we obtain $G(\gamma) \leq G(0)$. Therefore, we know that 
\begin{align}
J(\pi^*) - J(\hat{\pi}) & \leq G(0) \\
&= \eta \EE_{\rho \times \pi^*}\Big[ \big(\fps(s,a) - b(s) - {\fgt}(s,a)\big)^2 \Big] \notag \\
&\leq 4 \eta \Big( \EE_{\rho\times \pi^*} \big[\Gamma_n^2(s,a)\big] \wedge C^{\pi^*} \EE_{\rho \times \pi^*}\big[ \big(\fps(s,a) - b(s) - {\fgt}(s,a)\big)^2 \big] \Big) \notag\\
& = 4\eta \Big( \beta^2 \EE_{\rho\times \pi^*}\big[D^2_{\fcl}((s,a); \piref) \big] \wedge C^{\pi^*} \EE_{\rho \times \pi^*}\big[ \big(\fps(s,a) - b(s) - {\fgt}(s,a)\big)^2 \big] \Big)\notag \\
&= \tilde{O}\big(\eta D_{\pi^*}^2 \log{\fcl}(\epsilon_c) n^{-1}\big), \label{eq:final-rate}
\end{align}
where the inequality holds due to the definition of $\confduelbandit(\delta)$. Plugging~\eqref{eq:final-rate} into~\eqref{eq:dueling-taylor-expansion}, we know that $J(\pi^*) - J(\hat{\pi})$ has upper bound $\tilde{O}(D_{\pi^*}^2 n^{-1})$. By Lemma~\ref{lem:dueling-bandit:conf}, event $\confduelbandit$ with probability at least $1-\delta$, which concludes the proof.
\end{proof}



\subsection{Proof of Theorem~\ref{thm:lowerbound-dueling}}
\label{app:proof-thm-lower-dueling}

\begin{proof}[Proof of Theorem~\ref{thm:lowerbound-dueling}]
The proof is similar to the proof of Theorem~\ref{thm:lowerbound-bandit}. Consider the following family of contextual dueling bandit instances with $S \coloneqq |\cS|, A \coloneqq |\cA| < \infty$ and reward in some function class $\cG$.
\begin{align}\label{eq:dueling-bandit-class}
    \mathrm{CDB} \coloneqq \{(\cS, \cA, \rho, r, \piref, \eta): r \in \cG, \rho \in \Delta(\cS), \piref \in \Delta(\cA|\cS) \}. 
\end{align}
Fixing any $S \geq 1$, $\eta > 4\log 2$ and $C^* \in (2, \exp(\eta/4)]$, we aim to prove that, for any estimator $\cD \mapsto \hat\pi \in \Delta(\cA|\cS)$, for any $n \geq 16SC^*$, there exist some function class $\cG$, such that $\exists\ \mathrm{inst} = (\cS, \cA, \rho, r, \piref, \eta) \in \mathrm{CDB}$ with single-policy concentrability $C^{\pi^*} \leq C^*$, regularization coefficient $\eta$, $|\cS| = S = \Theta(\log|\cG|)$, and 
\begin{align}
     \inf_{\mathrm{inst} \in \mathrm{CDB}} \subopt_{\mathrm{RKL}}(\hat\pi; \mathrm{inst}) \gtrsim \min\{\eta SC^* n^{-1}, (SC^*)^{1/2}n^{-1/2}\}.
\end{align}    
Since $\log|\cG| \geq \log \cN_{\cG}(\epsilon)$ for any $\epsilon \in (0,1)$, the above bound yields the desired result. 

We construct the same reward function class as in the proof of Theorem~\ref{thm:lowerbound-bandit}. In particular, we set $\cS = [S]$, $\cA = \{\pm1\}$, $\rho = \mathsf{Unif}(\cS)$, and the reference policy to be
\begin{align*}
  \forall s \in \cS, \piref(-1|s) = C^{-1}, \piref(+1|s) = 1-C^{-1};
\end{align*}
where $C = C^*$. Then the total sub-optimality of any $\pi \in \Delta(\cA|\cS)$ given any reward function $r: \cS \times \cA \to \RR$ is
\begin{align}
    \suboptfdiv(\pi; r) = \frac{1}{S}\sum_{s=1}^S \suboptfdiv\big( \pi(\cdot|s); r(s, \cdot) \big). \label{eq:rkl-dueling-decomp}
\end{align}
We further let $\alpha = \eta^{-1}\log (C-1) \Leftrightarrow C - 1 = \exp(\eta \alpha)$. We construct $2^S$ Bernoulli reward functions, in particular, $\forall \tau \in \{\pm1\}^S$, the mean function $r_\tau$ of the reward (indexed by $\tau$) is defined as
\begin{align*}
    r_\tau(s, -1) = 0.5 + \tau_s\delta, r_\tau(s, +1) = 0.5 - \alpha.
\end{align*}
Then, following the derivation of  \Cref{eq:rkl:lb:case1,eq:rkl:lb:case2}, we know that $\forall s \in \cS$, $\forall \tau, \tau' \in \{\pm1\}^{\cS}$ with $\tau \sim_s \tau'$,
\begin{align}
    \subopt_s(\hat\pi; \tau) + \subopt_s(\hat\pi; \tau') \geq \frac{\eta\delta^2}{8} \wedge \frac{3\delta}{10}. \label{eq:rkl-dueling-single-summary}
\end{align}
Let $P_r$ be the distribution of $(s, a^1, a^2, y)$ for $s \sim \rho$, $a^1, a^2\iidsim \piref(\cdot|s)$ and $y \sim \mathsf{Bern}(\sigma(r(s,a^1) - r(s, a^2)))$. Now we set $\delta = \sqrt{S/n}$ and conclude that for $\tau \sim \tau'$ with $\tau_s = -\tau_s'$,
\begin{align*}
    & \kl{P_{r_\tau}}{P_{r_{\tau'}}} \\
    & \quad = \frac{(C-1)}{SC^2} \sum_{s', a^1, a^2} \kl{\mathsf{Bern}(\sigma(r_\tau(s',a^1) - r(s', a^2)))}{\mathsf{Bern}(\sigma(r_{\tau'}(s',a^1) - r(s', a^2)))} \\
    &\quad = \frac{2(C-1)}{SC^2} \Big( \kl{\mathsf{Bern}(\sigma(\alpha + \delta))}{\mathsf{Bern}(\sigma(\alpha - \delta))} \vee \kl{\mathsf{Bern}(\sigma(\alpha - \delta))}{\mathsf{Bern}(\sigma(\alpha + \delta))} \Big). 
\end{align*}
Since $\alpha, \delta \in (0, 1/2)$, by the fact $\KL(P\|Q) \leq 2 Q_{\min}^{-1} \tv{P}{Q}^2$ (see e.g., \citet[Section~7.6]{polyanskiy2025information}), we know that
\begin{align}
\kl{P_{r_\tau}}{P_{r_{\tau'}}} & \leq \frac{2(C-1)}{SC^2} \frac{4}{1 + \exp(\alpha + \delta)} \bigg( \frac{1}{1 + \exp(\alpha - \delta)} - \frac{1}{1 + \exp(\alpha + \delta)} \bigg)^2 \notag \\
    & \leq \frac{4}{3SC}\frac{\exp(2\alpha)(\exp(\delta) - \exp(-\delta))^2}{(1 + \exp(\alpha - \delta))^4} \notag \\
    & \leq \frac{4e}{3SC}(\exp(\delta) - \exp(-\delta))^2 \notag \\
    & \leq 36S^{-1}C^{-1}\delta^2, \label{eq:rkl-dueling-single}
\end{align}
where the second and third inequality hold due to $\alpha, \delta \leq 1/2$, and last inequality follows from $\exp(x) - \exp(-x) \leq 3x$ for $x \in [0, 1/2]$. Now we set $\delta = \sqrt{SC/n} \leq 1/4$. We substitute \Cref{eq:rkl-dueling-single-summary} into Assouad's Lemma (\Cref{lem:assouad}) and obtain that
\begin{align*}
     \inf_{\mathrm{inst} \in \mathrm{CDB}} \subopt_{\mathrm{RKL}}(\hat\pi; \mathrm{inst}) & \geq  \frac{1}{4} S \cdot \frac{1}{S} \cdot \bigg(\frac{\eta\delta^2}{8} \wedge \frac{3\delta}{10}\bigg) \cdot \min_{\tau \sim \tau'} \exp\Big( - \kl{P_{\cD_\tau}}{P_{\cD_{\tau'}}} \Big) \notag\\
    &= \frac{1}{4} \bigg(\frac{\eta\delta^2}{8} \wedge \frac{3\delta}{10}\bigg) \exp\Big( - n\kl{P_{r_\tau}}{P_{r_{\cD_{\tau'}}}} \Big) \\
    & \geq \frac{\exp(-36)}{32}\min\{\eta CSn^{-1}, S^2C^2 n^{-2} \},
\end{align*}
where the $1/S$ comes from the denominator of~\eqref{eq:rkl-dueling-decomp} and the second inequality follows from \Cref{eq:rkl-dueling-single}.
\end{proof}



\subsection{Proof of Theorem~\ref{thm:dueling-bandit-f-div-upper}}

\begin{proof}[Proof of Theorem~\ref{thm:dueling-bandit-f-div-upper}]
The proof is similar to the proof of Theorem~\ref{thm:f-upper-bound}. Recall that $b(\cdot)$ defined in~\eqref{eq:def-bias-dueling}, we know that
\begin{align*}
    \hat \pi & = \argmax_{\pi \in \Delta^d} \Big\{\EE_{(s,a) \sim \rho \times \pi}[\fls(s,a)] - \eta^{-1}\EE_{s \sim \rho}\big[D_{f}(\pi || \piref)\big] \Big\} \\
    & = \argmax_{\pi \in \Delta^d} \Big\{\EE_{(s,a) \sim \rho \times \pi}[\fls(s,a) - b(s)] - \eta^{-1}\EE_{s \sim \rho}\big[D_{f}(\pi || \piref)\big] \Big\}.
\end{align*}
We have the following sub-optimality decomposition
\begin{align*}
    J(\pi^*) - J(\hat \pi) &= \EE_{s \sim \rho} \Big[\EE_{a \sim \pi^*}[\fgt(s,a)] - \EE_{a \sim \hat \pi}[\fgt(s,a)] - \eta^{-1}\big[D_{f}(\pi^* \| \piref) - D_{f}(\hat \pi \| \piref) \big]\Big]  \\
    & = \EE_{s \sim \rho}\big[H^*_s(\fgt) - H^*_s(\fls - b) - \la \hat \pi, \fgt - \fls +b \ra \big] \\
    & =  \EE_{s \sim \rho}\big[H^*_s(\fgt) - H^*_s(\fls - b) - \la \nabla H^*_s(\fls - b), \fgt - \fls + b \ra \big] \\
    & = \EE_{s \sim \rho}[(\fgt - \fls + b)^\top \nabla^2H^*_s(\tilde g)(\fgt - \fls + b)],
\end{align*}
where $\tilde g = \gamma \fgt + (1-\gamma)\fls$ and $\gamma \in [0,1]$, $(\fls - b)(s,a) = \fls(s,a) - b(s)$ and the last equation holds due to Taylor's expansion. Now, for any $\delta \in (0,1)$ and $\epsilon_c > 0$,  with probability at least $1 - \delta$
\begin{align*}
    J(\pi^*) - J(\hat \pi) & = \EE_{s \sim \rho}[(\fgt - \fls + b)^\top \nabla^2H^*_s(\tilde g)(\fgt - \fls + b)] \\
    & \leq \alpha^{-1} \eta \EE_{s \sim \rho}\Big[(\fgt - \fls + b)^\top \mathrm{diag}\big(\piref(a_1|s), \cdots,\piref(a_d|s)\big)(\fgt - \fls + b)\Big] \\
    & = \alpha^{-1} \eta \EE_{(s,a) \sim \rho \times \piref}\big[\big(\fgt(s,a) - \fls\big(s,a) + b(s))^2\big] \\
    & \leq \alpha^{-1} \eta \bigg( \frac{128}{3n}\log(2\cN_{\fcl}(\epsilon_c)/\delta) + 18 \epsilon_c \bigg),
\end{align*}
where the first inequality holds due to Lemma~\ref{lem:super-hessian} and last inequality holds due to equation~\eqref{eqn:dueling-conf:1}. Setting $\epsilon_c = O(n^{-1})$ completes the proof.
\end{proof}

\subsection{Proof of Theorem~\ref{thm:dueling-bandit-f-div-lower}}

\begin{proof}[Proof of Theorem~\ref{thm:dueling-bandit-f-div-lower}] 
We still consider the contextual dueling bandit instance class defined in~\eqref{eq:dueling-bandit-class}. We show that given any positive $\alpha, \eta$, for any $n \geq S \cdot \max\{ 16, \eta^2\alpha^{-2}\}$, there exists $f: \RR \to \RR$ such that $f$ is $\alpha$-strongly convex, $\log |\cG| = \Theta(S)$ and
\begin{align}
    \inf_{\hat\pi \in \hat\Pi(\cD)} \sup_{\mathrm{inst} \in \mathrm{CDB}} \suboptfdiv(\hat\pi; \mathrm{inst}) \gtrsim \frac{\eta S}{\alpha n}, \label{eq:f-dueling-lb}
\end{align}
where $\cD = \{(s_i, a_i^1, a_i^2, y_i)\}_{i=1}^n$ is the offline preference dataset, all (possibly randomized) maps from which to $\Delta(\cA|\cS)$ is denoted by $\hat\Pi(\cD)$. Since $S = \Theta(\log |\cG|) \gtrsim \log \cN_{\cG}(\epsilon_c)$ for all $\epsilon_c \in (0,1)$, we can conclude the theorem.


Let $\cS = [S]$, $\cA = \{\pm 1\}$, $\rho = \mathsf{Unif}(\cS)$ and $\piref(\cdot|s) = \mathsf{Unif}(\cA)$ for any $s \in \cS$. Then the total sub-optimality of any $\pi \in \Delta(\cA|\cS)$ given any reward function $r: \cS \times \cA \to \RR$ is
\begin{align}
    \suboptfdiv(\pi; r) = \frac{1}{S}\sum_{s=1}^S \suboptfdiv\big( \pi(\cdot|s); r(s, \cdot) \big). \label{eq:f-dueling-decomp}
\end{align}
We still consider the reward function class $\cG$ indexed by $\{\pm 1\}^{\cS}$. For all $\tau \in \{\pm 1\}^{\cS}$ the reward instance ``shaped'' by $\tau$ is
\begin{align}
   r_\tau(s, a) = \frac{1}{2} + a\tau_s \cdot\sqrt{\frac{S}{n}},  
\end{align}
where $a\tau_s = \pm 1$ because $a \in \cA = \{\pm1\}$. We thereby refer $\tau \sim \tau'$ to any pair in $\{\pm 1\}^{\cS}$ that differs only in one coordinate. $\forall \tau, \tau' \in \{\pm 1\}^{\cS}$, if $\tau \sim \tau'$, then suppose $\tau_s = - \tau_s'$, we have
\begin{align}
    \suboptfdiv(\pi(\cdot|s); r_\tau(s, \cdot)) + \suboptfdiv(\pi(\cdot|s); r_{\tau'}(s, \cdot)) \geq \frac{\eta S}{\alpha n}, \label{eq:f-dueling-delta}
\end{align}
where the inequality follows from exactly the same calculation in equation~\eqref{eq:fdiv:pre-lb} by setting $f(x) = \alpha(x-1)^2/2$.\footnote{Recall that in this case $D_f = {\chi^2}$, where $2{\chi^2}(\mu\|\nu) + 1= \sum_{a\in\cA}[\mu(a)]^2/\nu(a)$.} 
Let $P_r$ be the distribution of $(s, a^1, a^2, y)$ for $s \sim \rho$, $a^1, a^2\iidsim \piref(\cdot|s)$ and $y \sim \mathsf{Bern}(\sigma(r(s,a^1) - r(s, a^2)))$. Then we denote $\delta = \sqrt{S/n}$ and conclude that for $\tau \sim \tau'$ with $\tau_s = -\tau_s'$,
\begin{align}
    \kl{P_{r_\tau}}{P_{r_{\tau'}}} &= \frac{1}{SA^2} \sum_{s', a^1, a^2} \kl{\mathsf{Bern}(\sigma(r_\tau(s',a^1) - r(s', a^2))}{\mathsf{Bern}(\sigma(r_{\tau'}(s',a^1) - r(s', a^2))} \notag\\
    &= \frac{1}{4S} \Big( \kl{\mathsf{Bern}(\sigma(2\delta))}{\mathsf{Bern}(\sigma(-2\delta))} + \kl{\mathsf{Bern}(\sigma(-2\delta))}{\mathsf{Bern}(\sigma(2\delta))} \Big) \notag\\
    &\leq \frac{1}{4S} \Big( (\exp(-2\delta) -1)^2 + (\exp(2\delta) -1)^2 \Big) \notag\\
    &\leq \frac{1}{2S}(\exp(2\delta) -1)^2 \leq \frac{36\delta^2}{2S} = \frac{18}{n}, \label{eq:f-dueling-single}
\end{align}
where the last inequality follows from $\exp(x) - 1 \leq 3x$ for $x \in [0, 0.5]$ and $\delta = \sqrt{S/n} \leq 0.25$ by assumption. Therefore, we substitute \Cref{eq:f-dueling-delta} into Assouad's Lemma (\Cref{lem:assouad}) to obtain
\begin{align}
    \text{LHS of \Cref{eq:f-dueling-lb}} &\geq \frac{1}{S} \cdot S \cdot \frac{\eta S}{\alpha n} \cdot \frac{1}{4} \cdot \min_{\tau \sim \tau'} \exp\Big( - \kl{P_{\cD_\tau}}{P_{\cD_{\tau'}}} \Big) \notag\\
    &= 0.25 \cdot \frac{\eta S}{\alpha n} \cdot \exp\Big( - n\kl{P_{r_\tau}}{P_{r_{\cD_{\tau'}}}} \Big)  \geq \frac{\eta S}{\alpha n}  \cdot \frac{1}{3} \cdot \exp(-18) \gtrsim \frac{\eta S  }{\alpha n},
\end{align}
where the $1/S$ comes from the denominator of \Cref{eq:f-dueling-decomp} and the second inequality follows from \Cref{eq:f-dueling-single}.
\end{proof}

\section{Auxiliary Lemmas}

\begin{lemma}[{\citealt[Lemma~D.4]{zhao2024sharp}}]\label{lem:concentration-dueling}
Consider a offline dataset $\{(s_i,a^1_i, a^2_i, y_i)\}_{i=1}^n$ generated from the product of the context distribution $\rho \in \Delta(\cS)$, policy $\pi \in \Delta(\cA|\cS)$, and the Bradley-Terry Model defined in \Cref{sec:duel-setup}. Suppose $\fls$ is the result of MLE estimation of Algorithm~\ref{algorithm:dueling-bandit}, and we further define $b(s) = \EE_{a\sim \pi(\cdot|s)}\big[\fls(s,a) - \fgt(s,a)\big]$, then with probability at least $1 - 2 \delta$, we have
\begin{align*}
    \EE_{s,a \sim \rho\times \pi}\big[\big(\fls(s, a) - \fgt(s,a) - b(s)\big)^2\big] \leq O\bigg(\frac{1}{n}\log(\cN_{\fcl}(\epsilon_c)/\delta) + \epsilon_c\bigg).
\end{align*}
\end{lemma}

\Cref{lem:fano,lem:assouad} are two standard reductions \citep{lecam1973convergence,yu1997assouad,polyanskiy2025information}. See, e.g., \citet[Section~3]{chen2024assouad} for a general proof.

\begin{lemma}[Fano's inequality]\label{lem:fano}
Fix any $\cR \coloneq \{r_1, \cdots, r_S\}$ and policy class $\Pi$, let $L : \Pi \times \cR \to \RR_{+}$ be some loss function. Suppose there exist some constant $c > 0$ such that the following condition holds:
\begin{align*}
    \min_{i \neq j} \min_{\pi \in \Pi} L(\pi, r_i) + L(\pi, r_j) \geq c.
\end{align*}
Then we have
\begin{align*}
    \inf_{\pi \in \Pi} \sup_{r \in \cR} \EE_{\cD \sim P_r} L\big( \pi(\cD), r \big) \geq \frac{c}{2} \bigg(1 - \frac{\max_{i \neq j}\KL(P_{r_i} \| P_{r_j}) + \log 2}{\log S} \bigg),
\end{align*}
where the trajectory distribution of $\pi$ interacting with instance $r \in \cR$ is denoted by $P_r$.
\end{lemma}

\begin{lemma}[Assouad's Lemma]\label{lem:assouad}
   Let $\cR$ be the set of instances, $\Pi$ be the set of estimators, $\Theta \coloneqq \{\pm1\}^S$ for some $S > 0$, and $\{L_j\}_{j=1}^S$ be $S$ functions from $\Pi \times \cR$ to $\RR_+$. Suppose $\{r_\theta\}_{\theta \in \Theta} \subset \cR$ and the loss function is
\begin{align*}
    L(\pi, r) \coloneqq \sum_{j=1}^S L_j(\pi, r), \forall (\pi, r) \in \Pi \times \cR.
\end{align*}
We denote $\theta \sim_j \theta'$ if they differ only in the $j$-th coordinate. Further assume that
\begin{align}
    \theta \sim_j \theta' \Rightarrow \inf_{\pi\in\Pi} L_j(\pi, r_\theta) + L_j(\pi, r_{\theta'}) \geq c
\end{align}
for some $c > 0$, then
\begin{align*}
    \inf_{\pi \in \Pi} \sup_{r \in \cR} \EE_{\cD \sim P_r} L\big( \pi(\cD), r \big) \geq S \cdot \frac{c}{4} \min_{\exists j: \theta \sim_j \theta'} \exp\Big( - \kl{P_{r_\theta}}{P_{r_{\theta'}}} \Big),
\end{align*}
where the trajectory distribution of $\pi$ interacting with instance $r \in \cR$ is denoted by $P_r$.
\end{lemma}

The following \Cref{lem:max-signals} is due to \citet{gilbert1952comparison,varshamov1957estimate}, which is a classical result in coding theory.

\begin{lemma}\label{lem:max-signals}
Suppose $\Sigma$ is a set of characters with $|\Sigma| = q$ where $q \geq 2$ is a prime power and $N>0$ is some natural number. Then there exists a subset $\cV$ of $\Sigma^{N}$ such that (1) for any $v, v' \in \cV, v \neq v_j$ , one has $d_H(v, v') \geq N/2$ and (2) $\log_q|\cV| \geq H_q(1/2) = \Theta(1)$, where $d_H$ is the Hamming distance and the entropy function $H$ is given by
\begin{align*}
    H_q(x) = x\frac{\log (q-1)}{\log q} - x\frac{\log x}{\log q} - (1-x) \frac{\log (1-x)}{\log q}.
\end{align*}
For example, when $q=2$, this means that there exists a subset $\cV$ of $\{-1, 1\}^S$ such that (1) $|\cV| \geq \exp(S/8)$ and (2) for any
$v, v' \in \cV, v \neq v_j$ , one has $\|v - v'\|_1 \geq S/2$.
\end{lemma}

\end{document}